\documentclass{article}

\usepackage{arxiv}
\usepackage{amsthm}
\usepackage[utf8]{inputenc} % allow utf-8 input
\usepackage[T1]{fontenc}    % use 8-bit T1 fonts
\usepackage{hyperref}       % hyperlinks
\usepackage{url}            % simple URL typesetting
\usepackage{booktabs}       % professional-quality tables
\usepackage{amsfonts}       % blackboard math symbols
\usepackage{nicefrac}       % compact symbols for 1/2, etc.
\usepackage{microtype}      % microtypography
\usepackage{lipsum}		% Can be removed after putting your text content
\usepackage{graphicx}
\usepackage{natbib}
\usepackage{doi}

\usepackage{amsmath, amssymb, amsfonts} 
\usepackage{nicefrac}
\usepackage{mathtools}
\usepackage{bm, bbm}
\usepackage[scr=boondoxo,scrscaled=1.05]{mathalfa}

\usepackage{enumitem}

\usepackage{xcolor}
\usepackage{comment}
\usepackage{soul}

\usepackage{cleveref}
\usepackage{thmtools}
\usepackage{tikz}
\usepackage{thm-restate}
\usepackage{blindtext}

\usepackage{afterpage}

\usepackage{algorithm, algorithmic}

\usepackage{dsfont}

\usepackage{tikz}
\usetikzlibrary{tikzmark,calc,decorations.pathreplacing}
\usetikzlibrary{shapes,positioning,fit,calc}
\usetikzlibrary{arrows}
\usepackage{graphicx}
\usepackage{tikzscale}
\usepackage{pgfplots}
\usepgfplotslibrary{fillbetween}
\usetikzlibrary{patterns}
\pgfplotsset{compat=1.5.1}
\pgfplotsset{compat=newest}
\usetikzlibrary{plotmarks}
\usepackage{xfrac}

\usepackage{thm-restate}

\usepackage{subcaption}
\usepackage{hhline}

\definecolor{orange}{rgb}{1,0.4,0.0}

\DeclarePairedDelimiterXPP{\KL}[2]{D_{\textnormal{KL}}}{(}{)}{}{%
#1\:\delimsize\|\:#2% 
}

\DeclarePairedDelimiterXPP{\RD}[2]{D_{$\alpha$}}{(}{)}{}{%
#1\:\delimsize\|\:#2%
}

\DeclarePairedDelimiterXPP\Prob[1]{\mathbb{P}}{\lbrace}{\rbrace}{}{

#1}

\DeclarePairedDelimiterXPP{\lnorm}[2]{}{\lVert}{\rVert}{_{#2}}{#1}

\newcommand{\bE}{\ensuremath{\mathbb{E}}}

\newcommand{\bP}{\ensuremath{\mathbb{P}}}

\newcommand{\bR}{\ensuremath{\mathbb{R}}}

\newcommand{\bV}{\ensuremath{\mathbb{V}}}

\newcommand{\cA}{\ensuremath{\mathcal{A}}}

\newcommand{\cN}{\ensuremath{\mathcal{N}}}
\newcommand{\cO}{\ensuremath{\mathcal{O}}}

\newcommand{\mi}{\textup{I}}
\newcommand{\ent}{\textup{H}}

\newcommand{\argmax}{\ensuremath{\textnormal{argmax}}}
\newtheorem{theorem}{Theorem}
\newtheorem{mylemma}[theorem]{Lemma}

\newtheorem{myproposition}[theorem]{Proposition}
\newtheorem{mydefinition}[theorem]{Definition}

\newtheorem{myremark}[theorem]{Remark}

\newcommand{\innerproduct}[2]{\langle #1, #2 \rangle}

\newcommand{\PhiBeta}[2]{\ensuremath{\phi_\beta(\langle #1, #2 \rangle)}}

\newcommand{\Bern}[2]{\ensuremath{\textnormal{Bern}(\phi_\beta(\langle #1, #2 \rangle))}}

\newcommand{\Bernoulli}{\ensuremath{\textnormal{Bern}}}

\newcommand{\kl}{\textup{D}_{\textnormal{KL}}}

\crefname{theorem}{theorem}{theorems}
\crefname{corollary}{corollary}{corollaries}

\title{An Information-Theoretic Analysis of\\ Thompson Sampling for Logistic Bandits}

\author{%
  Amaury Gouverneur, Borja Rodríguez-Gálvez, Tobias J. Oechtering, and
  Mikael Skoglund \\
  KTH Royal Institute of Technology \\
  Stockholm, Sweden %\\
}
\date{\vspace{-5ex}}

\renewcommand{\shorttitle}{\textit{arXiv} Template}

\begin{document}

\maketitle

\begin{abstract}%
 We study the performance of the Thompson Sampling algorithm for logistic bandit problems. In this setting, an agent receives binary rewards with probabilities determined by a logistic function, $\exp(\beta \langle a, \theta \rangle)/(1+\exp(\beta \langle a, \theta \rangle))$, with slope parameter $\beta>0$, and where both the action $a\in \cA$ and parameter $\theta \in \cO$ lie within the $d$-dimensional unit ball. Adopting the information-theoretic framework introduced by~\citet{russo_information-theoretic_2015}, we analyze the information ratio, a statistic that quantifies the trade-off between the immediate regret incurred and the information gained about the optimal action.
We improve upon previous results by establishing that the information ratio is bounded by $\tfrac{9}{2}d\alpha^{-2}$, where $\alpha$ is a \emph{minimax measure} of the alignment between the action space $\cA$ and the parameter space $\cO$, and is independent of $\beta$. Using this result, we derive a bound of order $O(d/\alpha\sqrt{T \log(\beta T/d)})$ on the Bayesian expected regret of Thompson Sampling incurred after $T$ time steps. To our knowledge, this is the first regret bound for logistic bandits that depends only logarithmically on $\beta$ while being independent of the number of actions. In particular, when the action space contains the parameter space, 
the bound on the expected regret is of order $\tilde{O}(d \sqrt{T})$.
\end{abstract}

\section{Introduction}
\label{sec:introduction}
This paper studies the logistic bandit problem, where an agent sequentially interacts with an unknown environment with parameter $\Theta \in \cO$. At each time step, the agent selects an action $A_t \in \cA$ and receives binary rewards $R_t \in \{0,1\}$ with probabilities determined by the logistic function $\exp(\beta \langle A_t, \Theta \rangle) / (1 + \exp(\beta \langle A_t, \Theta \rangle))$ with slope parameter $\beta>0$. In this setting, both the action space and the parameter space lie within the $d$-dimensional unit ball. The goal of the agent is to maximize its total reward, or equivalently to \emph{minimize its regret}, that is the difference between the optimal cumulative reward and the cumulative reward achieved by the agent. This setting is used to model various scenarios, such as click-through rate prediction, spam email detection, and personalized advertisement systems, where, in the latter case, content is suggested to users who provide binary feedback ~\citep{chapelle_empirical_2011, russo_learning_2018}.\\

The performance, or regret, of algorithms for logistic bandits has been extensively studied, with significant contributions including analyses of Upper-Confidence-Bound (UCB) algorithms~\citep{filippi_parametric_2010,li_provably_2017,faury_improved_2020} as well as the study of Thompson Sampling (TS) ~\citep{russo_learning_2014,dong_performance_2019,abeille_linear_2017}. However, nearly all existing regret bounds for logistic bandits exhibit an exponential dependence on the parameter $\beta$ (see~\Cref{table:comparision}). This dependence is unsatisfactory as, in practice, the distinction between near-optimal and sub-optimal actions gets more pronounced as $\beta$ increases and it can be faster to find optimal actions. This lead \citet{mcmahan_open_2012} to make an open call for better bounds. \\

In this work, we focus on the Thompson Sampling algorithm~\citep{thompson_likelihood_1933}, which, despite its simplicity, has proven to be effective for a wide range of problems~\citep{russo_tutorial_2018, chapelle_empirical_2011}. To analyze the TS expected regret,~\citet{russo_information-theoretic_2015} introduced the concept of the information ratio, defined as the ratio between the squared expected instant regret%difference between the optimaland actual rewards
, and the information gained about the optimal action. Building on this framework, \citet{dong_information-theoretic_2018} derived a near-optimal regret rate of $O(d \sqrt{T \log T})$ for $d$-dimensional linear bandit problems. However, applying this analysis to the logistic setting yields regret bounds that grow exponentially with the parameter $\beta$ (see \Cref{sec:adaptation}).  Using numerical simulations, they conjectured that the TS information ratio for logistic bandits depends only on the problem's dimension $d$. %Since then, several works have aimed to characterize their claim. %ratio

\begin{table*}[!ht]
\begin{center}
\scalebox{0.88}{
\begin{tabular}{|c||c|c|}
 \hline
 \textbf{Algorithm} & \textbf{Regret Upper Bound} & \textbf{Note} \\ 
 \hhline{|=||=|=|}
 \begin{tabular}{c} Thompson Sampling\\ \citep{russo_learning_2014}\end{tabular} & $O\left(e^\beta\cdot d \cdot T^{1/2}\cdot\log( T)^{3/2} \right)$ & Bayesian bound\\
 \hline
 \begin{tabular}{c} GLM-TSL\\ \citep{abeille_linear_2017}\end{tabular} & $O\left(e^\beta\cdot d^{3/2}\cdot\log(d)^{1/2} \cdot T^{1/2}\log( T)^{3/2} \right)$ & Frequentist bound\\
 \hline
 \begin{tabular}{c} Thompson Sampling\\ \citep{dong_performance_2019}\end{tabular} & $O\left(e^\beta\cdot d\cdot T^{1/2}\log(T/d)^{1/2} \right)$ & Bayesian bound\\
 \hline
 \begin{tabular}{c} GLM-TSL\\ \citep{kveton2020randomized}\end{tabular} & $O\left(d\cdot T^{1/2}\cdot\log(T) + e^\beta \beta^6 d^2 \log(T)^2 \right)$ & Frequentist bound \\ 
 \hline
 \begin{tabular}{c} Logistic-UCB-2\\ \citep{faury_improved_2020}\end{tabular} & $O\left(d\cdot T^{1/2}\cdot\log(T) + e^\beta\cdot d^2 \cdot \log(T)^2 \right)$ & Frequentist bound \\ 
 \hline
 \begin{tabular}{c} {Thompson Sampling}\\ {(\emph{this paper})}\end{tabular} & $O\left(d/\alpha \cdot T^{1/2}\cdot\log( T\beta/d)^{1/2} \right)$ & \begin{tabular}{c} Bayesian bound,\\ $\alpha$ is independent of $\beta$%\\ (defined in Section \ref{sec:problem_setup})
 \end{tabular}\\
 \hline
\end{tabular}
}
\caption{Comparison of various regret guarantees for the logistic bandit problem. }
\label{table:comparision}
\end{center}
\end{table*}

Recently,~\citet{neu_lifting_2022} derived a regret bound of the order $O(\sqrt{dT|\cA|\log(\beta T)})$ on the logistic bandit problem, but their result relies on a worst-case information ratio bound scaling with the cardinality of the action space $|\cA|$ and their regret bound becomes vacuous for problems with continuous or infinite action space even though Thompson Sampling is known to perform well under these settings~\citep{russo_learning_2014}. Studying the TS regret for logistic bandits, \citet{dong_performance_2019} introduced two statistics to characterize the sets $\cA$ and $\cO$, the \emph{minimax alignment constant}\footnote{This statistic is referred to as the \emph{worst-case optimal log-odds} in the work of \citet{dong_performance_2019}.  }
$\alpha = \min_{\theta \in \cO} \max_{a \in \cA} \innerproduct{a}{\theta}$ and the \emph{fragility dimension} $\eta$, which is the cardinality of the largest subset of parameters such that their corresponding optimal action is misaligned with any other parameter from the subset. Using those statistics, they showed that for $\beta<2$, the TS information ratio is bounded by $100\max(d,\eta)\alpha^{-2}$. They also suggested, through numerical computations, that this bound holds for larger values of $\beta$. However, their work has two key limitations. First, they did not provide a rigorous proof for generalizing their bound to larger values of  $\beta$. Second, and more critically, their regret analysis is incorrect as it relies on the rate-distortion bound from~\citet{dong_information-theoretic_2018},
which is incompatible with a bound on the TS information ratio.  Indeed, the regret analysis in~\citet{dong_information-theoretic_2018} specifically requires a bound on the one-step compressed TS information ratio, which is a fundamentally different quantity from the TS information ratio studied in the work of \citet{dong_performance_2019}. 
We elaborate on these gaps in more detail in~\Cref{sec:arguments_gaps_in_dong}.\\
%Notably, their techniques to bound the information ratio do not apply to the one-step compressed TS information ratio.\\

In this paper, we address these issues and obtain a regret bound that scales only logarithmically with the slope of the logistic function, while remaining independent of the cardinality of the action space. Our key contributions are as follows:
\begin{itemize}
    \item  We prove an information-theoretic regret bound of order $O(\sqrt{T \Gamma (\ent(\Theta_\varepsilon)+\beta \varepsilon T)})$ that holds for infinite and continuous action and parameter spaces. The bound relies on the entropy of the parameter quantized at scale $\epsilon>0$, and on the average expected TS information ratio, $\Gamma$. %This result is achieved by adapting the approaches from~\citep{neu_information-theoretic_2021} and~\citep{gouverneur_thompson_2023} to the logistic bandit setting.
    \item We present a new analysis showing that, for all $\beta > 0$, the TS information ratio for logistic bandits is bounded by $\tfrac{9}{2}d\alpha^{-2}$, improving upon the previous results. Notably, our bound does not depend on the fragility dimension $\eta$ which can scale exponential in $d$.
    \item We establish a regret bound of order $O(d/\alpha\sqrt{T \log(\beta T/d)})$ for Thompson Sampling. To our knowledge, this is the first regret bound for any logistic bandit algorithm that scales only logarithmically with $\beta>0$ and is independent of the number of actions. Additionally, we show that if the action space encompasses the parameter space, the expected regret of Thompson Sampling is bounded in $O(d \sqrt{T \log(\beta T/d)})$ with no dependence on $\alpha$. 
\end{itemize}

The rest of the paper is organized as follows. \Cref{sec:problem_setup} introduces the logistic bandit problem, defines the Bayesian expected regret, and the specific notations used. \Cref{sec:information_ratio} introduces the Thompson Sampling algorithm and the information ratio analysis. \Cref{sec:main_results} states and discusses our main results, providing the improved regret bounds. \Cref{sec:analysis} presents the key ideas for analyzing the information ratio%.\Cref{sec:numerical_simulations} illustrates the improvement of our bounds compared to previous regret guarantees through numerical experiments
; and finally, Section~\ref{sec:conclusion} discusses our results and future extensions.

\section{Problem Setup}
\label{sec:problem_setup}
We consider a logistic bandit problem, where at each time step $t \in \{1,\ldots, T\}$, an agent selects an action $A_t \in \cA$ and receives a binary reward $R_t\in \{0,1\}$ with probability following a logistic function:
\begin{equation*}
    \bP(R_t = 1 | A_t = a, \Theta = \theta) = \frac{\exp(\beta \langle a, \theta \rangle)}{1 + \exp(\beta \langle a, \theta \rangle)}.
\end{equation*}
Here, $\beta>0$ is a known scale parameter, and $\langle a, \theta \rangle$ denotes the inner product of the action vector $a\in \cA$ and the unknown parameter $\theta\in\cO$. Throughout the paper, we denote the logistic function as $\phi_\beta(x) \coloneqq \frac{\exp(\beta x)}{1 + \exp(\beta x)}$. As this function is strictly increasing, the probability of obtaining a reward, $\phi_\beta(\innerproduct{a}{\theta})$, is maximized when the inner product between the action and parameter is maximized. In our setting, both the action $a$ and the parameter vector $\theta$ lie within the $d$-dimensional Euclidean unit ball, $\mathbf{B}_d(0,1)$\footnote{This setting is equivalent to the one considered by~\citet{faury_improved_2020} using $\beta$ as the maximal norm for $\theta \in \cO$.}. For a given action space $\cA$ and parameter space $\cO$, we define their \emph{minimax alignment constant} as $\alpha \coloneqq \min_{\theta \in \cO} \max_{a \in \cA} \innerproduct{a}{\theta}$. In the rest of the paper, we assume that the action and parameter spaces are such that $\alpha\geq0$. This assumption is relatively mild, as it suffices for $\cA$ to contain two opposed actions $a, a'$ (i.e. $a =-a'$) to ensure $\alpha \geq 0$ for any parameter set $\cO$.\\

Following the Bayesian framework, we assume the parameter vector $\Theta\in\cO$ is sampled from a known prior distribution $\bP_{\Theta}$. This prior, together with the reward distribution $\bP_{R|A, \Theta}$, fully describes the logistic bandit problem. As the reward distribution depends only on the selected action and the parameter, it can be written as $R_t = R(A_t,\Theta)$ for some random function $R:\cA\times\cO\to\bR$. 
The agent's history at time $t$ is denoted by $H^t = \{A_1, R_1, \ldots, A_{t-1}, R_{t-1}\}$, representing all past actions and rewards observed up to time $t$. The goal of the agent is to sequentially select actions that maximize the total cumulated reward, or equivalently, that minimize the total expected regret defined as:
\begin{equation*}
    \bE[\textnormal{Regret}(T)] \coloneqq \bE\left[ \sum_{t=1}^T R(A^\star, \Theta) - R(A_t, \Theta)\right],
\end{equation*}
where $A^\star$ is the \emph{optimal action}  corresponding to the parameter $\Theta$. We construct the optimal mapping $\pi_\star(\theta) \coloneqq \argmax_{a \in \cA} \bE[R(a, \theta)]$ so that we can write $A^\star = \pi_\star(\Theta)$. To ensure such a mapping exists, we make the technical assumption that the set of actions $\cA$ is compact. Following~\citet{dong_performance_2019}, we assume without loss of generality that the mapping $\pi_\star$ is one-to-one\footnote{%Recall that Thompson Sampling disregards actions that are not optimal for any parameter. 
If a particular parameter is optimal for multiple actions, we can arbitrarily fix the mapping of that parameter to one of the optimal actions. Conversely, if a particular action is optimal for multiple parameters, we can introduce duplicate action labels to ensure a one-to-one correspondence between each parameter and its optimal action label. A rigorous explanation of this construction is provided in~\Cref{sec:constructing_the_one_to_one_mapping}.}.\\ 

Since $\sigma$-algebras of the history are often used in conditioning, we introduce the notations $\bE_t[\cdot] \coloneqq \bE[\cdot | H^t]$ and $\bP_t[\cdot] \coloneqq \bP[\cdot | H^t]$ to denote the conditional expectation and probability given the history $H^t$, respectively. Additionally, we define $\mi_t(A^\star; R_t|A_t) \coloneqq \bE_t[\kl(\bP_{R_t | H^t, A^\star,A_t} \| \bP_{R_t | H^t,A_t})]$ as the disintegrated conditional mutual information between the optimal action $A^\star$ and the reward $R_t$ conditioned on the action $A_t$, \emph{given the history} $H^t$.

\section{Thompson Sampling and the Information Ratio}
\label{sec:information_ratio}
An elegant algorithm for solving bandit problems is the \emph{Thompson Sampling} algorithm. It works by randomly selecting actions according to their posterior probability of being optimal. More specifically, at each time step $t \in \{1, \ldots, T\}$, the agent samples a parameter estimate $\hat{\Theta}_t$ from the posterior distribution conditioned on the history $H^t$ and selects the action that is optimal for the sampled parameter estimate, $A_t = \pi_\star(\hat{\Theta}_t)$. The pseudocode for the algorithm is given in Algorithm~\ref{alg:Thompson_Sampling}.

\begin{algorithm}[ht]
    \caption{Thompson Sampling algorithm}
    \label{alg:Thompson_Sampling}
    \begin{algorithmic}[1]
        \STATE {\bfseries Input:} parameter prior $\bP_{\Theta}$, mapping $\pi_\star$.
        \FOR{$t=1$ {\bfseries to} T}
            \STATE Sample a parameter estimate $\smash{\hat{\Theta}_t \sim \bP_{\Theta|H^t}}$.
            \STATE Take the corresponding optimal action $A_t = \pi_\star(\hat{\Theta}_t)$.
            \STATE Collect the reward $R_t = R(A_t, \Theta)$.
            \STATE Update the history $H^{t+1} = H^t \cup \{A_t, R_t\}$.
        \ENDFOR
    \end{algorithmic}
\end{algorithm}

Studying the regret of Thompson Sampling,~\citet{russo_information-theoretic_2015} introduced a key quantity to the analysis, the \emph{information ratio} defined as the following random variable:
\begin{equation*}
%\label{eq:information_ratio}
    \Gamma_t \coloneqq \frac{\bE_t[ R(A^\star,\Theta) - R(A_t,\Theta)]^2}{\mi_t(A^\star; R(A_t,\Theta),A_t)}.
\end{equation*}
 This ratio measures the trade-off between minimizing the current squared regret and gathering information about the optimal action; a small ratio indicating that a substantial gain of information about the optimal action compensated for any significant regret.~\citet{russo_information-theoretic_2015} use this concept to provide a general regret bound in $O(\sqrt{\Gamma T \ent(A^\star)})$, that depends on the time horizon $T$, the prior entropy of the optimal action $\ent(A^\star)$, and an algorithm- and problem-dependent upper bound on the average expected information ratio $\Gamma$.\\

A limitation of this approach is that the prior entropy of the optimal action, $\ent(A^\star)$, can grow arbitrarily large with the number of actions and gets infinite with continuous action space. We address this issue in~\Cref{thm:quantized_bound}, where we propose a regret bound depending instead on the entropy of a quantization of the parameter $\Theta$ at scale $\epsilon>0$.

\section{Main Results}
\label{sec:main_results}
This section presents our main results on the Thompson Sampling regret for logistic bandits. In~\Cref {thm:quantized_bound}, we derive an information-theoretic regret bound for logistic bandits that holds for continuous and infinite parameter spaces. Following this, we state in~\Cref{prop:info_ratio_logistic_bandits} our principal contribution, a bound on the TS information ratio depending only on the problem's dimension $d$ and on the minimax alignment constant $\alpha$. By combining this result with our regret bound, we derive in~\Cref{thm:main_theorem}, a bound on the expected regret of TS for logistic bandits, which scales as $O(d/\alpha\sqrt{T \log(\beta T/d)})$.  \\

Our first theorem provides a regret bound that holds for large and continuous action spaces and relies on the entropy of the quantized parameter $\Theta_\varepsilon$, defined in~\Cref{def:quantized_parameter}, which is as the closest approximation for $\Theta$ (as measured by the metric $\rho$) on an $\varepsilon$-net for $(\cO, \rho)$. Compared to~\citet[Theorem~1]{dong_information-theoretic_2018}, our result is compatible with bounds on the ``standard" TS information ratio, rather than the \emph{``one-step compressed TS"} information ratio. This distinction is crucial, as the latter is significantly more challenging to analyze in logistic bandits due  to its intricate construction. 

%Our first theorem provides a regret bound for large and continuous action spaces that is compatible with the bound on the TS information ratio. It relates the regret of TS to the entropy of the quantized parameter $\Theta_\varepsilon$ defined as the closest approximation (measured by the metric $\rho$) for $\Theta$ on an $\varepsilon$-net for $(\cO,\rho)$.

\begin{restatable}{mydefinition}{definitionQuantizedParameter}
\label[definition]{def:quantized_parameter}
    Let the set $\cO_\varepsilon$ be an $\varepsilon$-net for $(\cO,\rho)$ with associated projection mapping $q:\cO \rightarrow \cO_\varepsilon$ such that for all $\theta \in \cO$ we have $\rho(\theta,q(\theta))\leq \varepsilon$. We define the \emph{quantized parameter} as $\Theta_\varepsilon \coloneqq q(\Theta)$.
\end{restatable}

\begin{restatable}{mytheorem}{theoremQuantized}
    \label[theorem]{thm:quantized_bound}
    For all $\beta>0$, under the logistic bandit setting with logistic function $\phi_{\beta}(x)$, let the quantized parameter $\Theta_\varepsilon$ be defined as in~\Cref{def:quantized_parameter} for some $\varepsilon>0$. If the average expected TS information ratio is bounded, $\frac{1}{T} \sum_{t=1}^T \bE[\Gamma_t] \leq \Gamma$, for some $\Gamma>0$, then the TS regret is bounded as
    \begin{equation*}
        \bE[\textnormal{Regret}(T)] \leq \sqrt{\Gamma T \left(\ent(\Theta_\varepsilon)+\varepsilon \beta T \right)} .
    \end{equation*}
\end{restatable}
The proof of~\Cref{thm:quantized_bound} adapts the techniques of~\citet[Theorem~2]{neu_lifting_2022} and~\citet[Theorem~2]{gouverneur_thompson_2023} to the logistic bandits setting. It relies on an approximation of the conditional mutual information $\mi(\Theta;R_t|A_t,H^t)$ as $ \mi(\Theta_\varepsilon;R_t|A_t,H^t) + \beta \varepsilon$ exploiting the fact that, for all $a \in \cA$, the log-likelihood of $R(a,\theta)$ is $\beta$-Lipschitz with respect to $\theta\in \cO$. \\

\begin{proof}
    We start by rewriting the TS expected regret using the information ratio: 
    \begin{align*}
        \bE[\textnormal{Regret}(T)] = \sum_{t=1}^T \bE[ R(A^\star,\Theta) - R(A_t,\Theta)]= \sum_{t=1}^T\bE\left[\sqrt{\Gamma_t \mi_{t} (A^\star;  R(A_t,\Theta),A_t)}\right].
    \end{align*}
    We continue using Jensen's inequality, followed by Cauchy-Schwarz inequality:
    \begin{align*}
 \bE[\textnormal{Regret}(T)]&\leq\sum_{t=1}^T\sqrt{\bE[\Gamma_t]\mi (A^\star;  R(A_t,\Theta),A_t|H^t)} \leq \sqrt{\Gamma T \sum_{t=1}^T\mi (A^\star;  R(A_t,\Theta),A_t|H^t)},
    \end{align*}
     where in the last inequality, we used $\sum_{t=1}^T \bE_t[\Gamma_t] \leq \Gamma T$. Applying the chain rule ~\citep[Theorem~3.7.b]{yury_polyanskiy_information_2022} we decompose the mutual information as 
     \begin{align*}
         \mi (A^\star;  R(A_t,\Theta),A_t|H^t) &= \mi (A^\star; A_t|H^t) + \mi_{t} (A^\star;  R(A_t,\Theta)|H^t,A_t)\\
         &=\mi (\pi_\star(\Theta);  R(A_t,\Theta)|H^t,A_t)\\
         &=\mi (\Theta;  R(A_t,\Theta)|H^t,A_t),
     \end{align*}
where we used the fact that the mutual information $\mi_{t} (A^\star; A_t|H^t) = 0 $ as $A^\star$ and $A_t$ are independent conditioned on the history $H^t$ and the last equality follows as $\pi_\star$ is a one-to-one mapping ~\citep[Theorem~3.2.d]{yury_polyanskiy_information_2022}.

Let $f_{R_t |H^t,  A_t, \Theta}$  denote the probability density of $R_t$ conditioned on $H^t,  A_t, \Theta$ and $f_{R_t |H^t,  A_t}$ denote the probability density on $H^t,  A_t$. Then, the mutual information terms can be written as
    \begin{align*}
        \mi(\Theta; R_t |H^t,  A_t) = \bE \bigg[ \log \frac{f_{R_t |H^t,  A_t, \Theta}(R_t)}{f_{R_t |H^t,  A_t}(R_t)} \bigg].
    \end{align*}

    We let the set $\cO_\varepsilon$ be an $\varepsilon$-net for $(\cO,\rho)$ with associated mapping $q:\cO \rightarrow \cO_\varepsilon$ and similarly to~\citet[Theorem~2]{neu_lifting_2022}, we observe that the mutual information can be written as
    \begin{align}
        \bE \bigg[ &\int_\cO  f_{\Theta|R_t,H^t,  A_t}(\theta) \bigg( \log \frac{f_{R_t |  A_t, \Theta=\theta}(R_t)}{f_{R_t |  A_t, \Theta = q(\theta)}(R_t)} + \log \frac{f_{R_t |H^t,  A_t, \Theta=q(\theta)}(R_t)}{f_{R_t |H^t,  A_t}(R_t)} \bigg) d\theta \bigg],\label{eq:mutual_information_decomposition}
    \end{align}
    using the fact that $f_{R_t |H^t,  A_t, \Theta} = f_{R_t |  A_t, \Theta}$ by the conditional Markov chain $R_t - A_t - H_t \  | \ \Theta$. 
    
    Since the derivative of $\log(\phi_\beta(x))$ is $\beta/(1+\exp(\beta x))$, it is bounded by $\beta$, which makes it $\beta$-Lipschitz. Furthermore, for all $a \in \cA$ and $\theta \in \cO$, the inner product $\innerproduct{a}{\theta} \leq 1$, implying that $\log(f_{R_t | A_t, \Theta=\theta}(1))$ is also $\beta$-Lipschitz with respect to $\theta$. Similarly, the derivative of $\log(1-\phi_\beta(x))$ is bounded by $\beta$, making it $\beta$-Lipschitz as well. Consequently, $\log(f_{R_t | A_t, \Theta=\theta}(0))$ is $\beta$-Lipschitz with respect to $\theta \in \cO$. Thus, we conclude that for all $\theta \in \cO$, we have $| \log f_{R_t| A_t, \Theta= \theta}(R_t) - \log f_{R_t| A_t, \Theta=q(\theta)}(R_t)| \leq \beta \rho(\theta, q(\theta)) \leq \beta \varepsilon$. 
    
    After defining the random variable $\Theta_\varepsilon \coloneqq q(\Theta)$, the second term in~\cref{eq:mutual_information_decomposition} is equal to $\mi(\Theta_\varepsilon; R_t|H^t,  A_t)$. 
    Summing the $T$ mutual information $\mi(\Theta_\varepsilon; R_t |H^t,  A_t)$ and applying the chain rule,  we obtain 
    \begin{align*}
            \bE[\textnormal{Regret}(T)]&\leq \sqrt{\Gamma T  \left(\mi(\Theta_\varepsilon; H^T) +  \varepsilon \beta T\right)}.
    \end{align*}
    Finally, we upper bound $\mi(\Theta_\varepsilon; H^T)$ by the entropy $\ent(\Theta_\varepsilon)$ to obtain the claimed result. 
\end{proof}

%Notably, the above theorem works with bounds on the average expected information ratio of the \emph{``standard''} TS, and not the one-step compressed TS as in~\citep[Theorem~1]{dong_information-theoretic_2018}. This distinction is crucial as the intricate construction of the one-step compressed TS makes it challenging to derive bounds on its information ratio.  %We explore this difference in more detail in~\Cref{sec:analysis}. 

In the following, we present our main proposition, a bound on the TS information ratio that depends only on the problem dimension $d$ and the minimax alignment constant $\alpha$.  

\begin{restatable}{myproposition}{propositionInformationRatio}
\label[proposition]{prop:info_ratio_logistic_bandits}
    For all $\beta>0$, and for all $\cA, \cO \subseteq \mathbf{B}_d(0,1)$ with minimax alignment constant $\alpha$, under the logistic bandit setting  with logistic function $\phi_{\beta}(x)$ the TS information ratio is bounded as 
    \begin{align*}
     \Gamma_t \leq \frac{9}{2} d\alpha^{-2}.   
    \end{align*}
    %$\Gamma_t \leq \frac{9}{2} d\alpha^{-2}$.
\end{restatable}

At a high level, our proof consists of three parts: a lower bound on the conditional mutual information, $\mi_t(A^\star; R(A_t,\Theta),A_t)$, an upper bound on the squared expected regret at time $t$, $\bE_t[ R(A^\star,\Theta) - R(A_t,\Theta)]^2$, and an upper bound on a ratio of expected variances by the study of the limit case $\beta \to \infty$. The key techniques for the proof of~\Cref{prop:info_ratio_logistic_bandits} are presented in~\Cref{sec:analysis}. \\

Combining~\Cref{prop:info_ratio_logistic_bandits} with~\Cref{thm:quantized_bound}, we arrive at our main result: a bound on the expected TS regret in $O(d/\alpha \sqrt{T \log(\beta T/d)})$. To the best of our knowledge, this is the first regret bound for logistic bandits that scales only logarithmically with the logistic function's parameter $\beta$ while remaining independent of the number of actions.

\begin{restatable}{mytheorem}{theoremMain}
\label[theorem]{thm:main_theorem}
     For all $\beta>0$, and for all $\cA, \cO \subseteq \mathbf{B}_d(0,1)$ with minimax alignment constant $\alpha$, under the logistic bandit setting  with logistic function $\phi_{\beta}(x)$, the TS regret is bounded as
     \begin{align*}
         \bE[\textnormal{Regret}(T)] \leq 3d/\alpha\sqrt{ T \log\left(\sqrt{3+\frac{6 \beta T}{d}}\right)}.
     \end{align*}
\end{restatable}
\begin{proof}
     Combining~\Cref{thm:quantized_bound} with ~\Cref{prop:info_ratio_logistic_bandits}, we upper bound the entropy $\ent(\Theta_\varepsilon)$ by the cardinality of the $\varepsilon$-net to get a regret bound of $3/\alpha\sqrt{d T/2\left(\log(|\Theta_\varepsilon|)  + \varepsilon \beta T\right)}$. To define $\Theta_\varepsilon$, we set $\cO_\varepsilon$ as the $\varepsilon$-net of smallest cardinality. As the parameter space $\cO$ is within the Euclidean unit ball, we use~\Cref{lemma:van_handel} to control the covering number as $\log(|\Theta_\varepsilon|) \leq d\log(1+2/\varepsilon)$ to upper bound the TS regret as
\begin{align*}
     \bE[\textnormal{Regret}(T)] \leq 3/\alpha \sqrt{d T/2 \left(d\log\left(1+\frac{2}{\varepsilon}\right)+\varepsilon \beta T \right)}.
\end{align*}
Finally, setting $\varepsilon = d/(\beta T)$ and rearranging terms inside the logarithm yields the desired result.%: 
% \begin{align*}
%     \bE[\textnormal{Regret}(T)] &\leq 3d/\alpha\sqrt{T/2 \left(\log\left(1+\frac{2 \beta T}{\beta}\right)+1 \right)}\\
%     &\leq 3d/\alpha\sqrt{T/2 \log\left(3+ \frac{6 \beta T}{\beta}\right)}=  3d/\alpha\sqrt{T \log\left(\sqrt{3+ \frac{6 \beta T}{\beta}}\right)}.
% \end{align*}
\end{proof}

Importantly, the above theorem does not depend on the fragility dimension $\eta$, in contrast to the results of~\citet{dong_performance_2019}. This distinction matters as, except in the case where $\smash{\alpha = 1}$, the fragility dimension can grow exponentially with the dimension $d$. We can verify that our result, due to its logarithmic dependence on $\beta$, is compatible with~\citet[Proposition 11]{dong_performance_2019}, which shows that there exist logistic bandit problems for which no algorithm can achieve a Bayesian regret uniformly bounded by $f(\alpha)p(d)T^{1-\epsilon}$, for some function $f$, polynomial $p$, and $\epsilon> 0$.\\

%However, removing the dependence on the minimax alignment constant $\alpha$ for the general case is not possible, as~\citep[Proposition 11]{dong_performance_2019} showed that there exist some logistic bandit problems such that no algorithm can achieve a Bayesian regret uniformly bounded by $f(\alpha)p(d)T^{1-\epsilon}$ for some function $f$, polynomial $p$, and $\epsilon > 0$.

%Nonetheless, this is possible in certain specific settings, and i
The next two corollaries present cases where the dependence on the minimax alignment constant $\alpha$ can be removed. The case in~\Cref{corollary:design} is particularly relevant for applications where the action set can be treated as a design parameter, and where constructing large action spaces is not prohibitive. We illustrate the improvement of~\Cref{corollary_alpha_one} over previous works through numerical experiments on a synthetic logistic bandit problem. The results are presented in~\Cref{sec:numerical_simulations}.

\begin{restatable}{mycorollary}{corollary_alpha_one}
\label[corollary]{corollary_alpha_one}
    For all $\beta>0$, under the logistic bandit setting  with logistic function $\phi_{\beta}(x)$, let $\smash{\cA \subseteq \mathbf{B}_d(0,1)}$ and $\smash{\cO \subseteq \mathbf{S}_d(0,1)}$ be such that $\smash{\cO \subseteq \cA}$.  Then the TS regret is bounded as
     \begin{align*}
         \bE[\textnormal{Regret}(T)] \leq 3d\sqrt{T \log\left(\sqrt{3+\frac{6 \beta T}{d}}\right)}.
     \end{align*}
\end{restatable}
\begin{proof}
    If $\cO \subseteq \mathbf{S}_d(0,1)$ and if $\cO \subseteq \cA$, then for each $\theta \in \cO$, there exists an action $a \in \cA$ such that $a \smash{=} \theta$ and $\innerproduct{a}{\theta}\smash{=}1$, implying $\alpha = 1$. Using~\Cref{thm:main_theorem} concludes the proof.
\end{proof}

\begin{restatable}{mycorollary}{corollary_designed_action_space}
\label[corollary]{corollary:design}
    For all $\beta>0$, under the logistic bandit setting  with logistic function $\phi_{\beta}(x)$, there exists an action space $\cA$ with $|\cA| \leq 2d\cdot 3^{d-1}$ such that for any $\cO \subseteq \mathbf{S}_d(0,1)$, the TS regret is bounded as
     \begin{align*}
         \bE[\textnormal{Regret}(T)] \leq 6d\sqrt{ T \log\left(\sqrt{3+\frac{6 \beta T}{d}}\right)}.
     \end{align*}
\end{restatable}
\begin{proof} 
Starting from~\Cref{thm:main_theorem}, we have to construct $\cA$ such that its minimimax alignment constant $\alpha$ is greater or equal to $\frac{1}{2}$ for any $\cO\subseteq \mathbf{S}_d(0,1)$. This is satisfied if $\cA$ is a $\frac{1}{2}$-net for $\mathbf{S}_d(0,1)$. Setting $\cA$ as the $\frac{1}{2}$-net of minimal cardinality, from~\Cref{lemma:polyanskiy}, we have $|\cA| \leq 2d\cdot 3^{d-1}$. 
\end{proof}

\section{Analysis}
\label{sec:analysis}
This section presents the main technique of the proofs of our main proposition,~\Cref{prop:info_ratio_logistic_bandits}. For the sake of clarity, we present here our results for the particular setting of~\Cref{corollary_alpha_one}, which ensures $\alpha =1$. We prove how to extend those results to general spaces in~\Cref{sec:extension}. Our proof can be divided into three parts: a lower bound on the mutual information (\Cref{subsec:lower_bound_mi}), an upper bound on the squared expected regret (\Cref{subsec:upper_bound_squared_regret}), and an upper bound on a ratio of expected variances by analyzing the limit case as $\beta \to \infty$ (\Cref{subsec:ratio_exp_variances}). To alleviate the notations, we omit the subscript $t$ for the rest of the section.

A key quantity to our analysis is the expected variance of the reward probability conditioned on the sampled action, expressed as $\bE[\bV[\phi_\beta(\innerproduct{\hat{A}}{\Theta})|\hat{A}]]$. We will use it to lower bound on mutual information and a related quantity is used to upper bound the squared expected regret. Intuitively, when the variance of reward probability is high, the agent is exploring new actions, gathering information about $\Theta$ but suffering regret. Conversely, if the variance of reward probability is low, it indicates that the agent has already identified near-optimal actions and is exploiting this knowledge.\\

Under the logistic bandit setting with logistic function $\phi_\beta$, the reward $R(A_t,\Theta)$ is given by a Bernoulli random variable with associated probability $\phi_\beta(\innerproduct{A_t}{\Theta})$. We
 denote it by $\Bern{A_t}{\Theta}$ to make the setting more explicit. With this notation, we write the information ratio as: 
\begin{equation*}
    \Gamma = \frac{\bE[ \Bern{A^\star}{\Theta} - \Bern{\hat{A}}{\Theta}]^2}{\mi(A^\star; \Bern{\hat{A}}{\Theta},\hat{A})}.
\end{equation*}

\subsection{Lower bounding the mutual information}
\label{subsec:lower_bound_mi}

We start by stating a general lemma that relates the variance of a $[0,1]$ random variable $U$ to the mutual information between $U$ and a Bernoulli outcome with probability $U$. The proof, presented in~\Cref{sec:proof_of_useful_lemmata}, uses the decomposition of mutual information as a difference of entropy and the Taylor expansion of the entropy function.

\begin{restatable}{mylemma}{LowerBoundOnMiForBernoulli}
\label[lemma]{lemma:LowerBoundOnMiForBernoulli}
     Let $U$ be a random variable taking values in $[0,1]$ and $\Bernoulli(U)$ be a Bernoulli random variable with probability $U$. Then it holds that,
    \begin{align*}
    \mi(U;\Bernoulli(U)) \geq 2 \bV(U).  
    \end{align*}
\end{restatable}

Using~\Cref{lemma:LowerBoundOnMiForBernoulli}, we prove a lower bound on the mutual information $\mi(A^\star; \Bern{\hat{A}}{\Theta},\hat{A})$ is lower bounded by the expected variance of reward probability $\bE[\bV[\phi_\beta(\innerproduct{\hat{A}}{\Theta})|\hat{A}]]$.

\begin{mylemma}
\label[lemma]{lemma:lower_bound_on_mi}
Let the logistic function be $\phi_{\beta}(x)$, then, it holds that 
\begin{align*}
    \mi(A^\star; \Bern{\hat{A}}{\Theta},\hat{A})\geq 2 \bE \left[ \bV\left[ \phi_{\beta}(\innerproduct{\hat{A}}{\Theta}) \mid \hat{\Theta} \right] \right].
\end{align*}
\end{mylemma}

\begin{proof}

We start by applying the chain rule. It comes that
\begin{align*}
    \mi(\pi_\star(\Theta); \pi_\star(\hat{\Theta}), \Bernoulli (\phi_{\beta}(\innerproduct{\hat{A}}{\Theta}))) &\overset{(i)}{=}\mi(\Theta; \hat{\Theta}, \Bernoulli (\phi_{\beta}(\innerproduct{\hat{A}}{\Theta})))\\
    &\overset{(j)}{=} \mi (\Theta; \hat{\Theta}) + \mi (\Theta; \Bernoulli (\phi_{\beta}(\innerproduct{\hat{A}}{\Theta}))  \mid \hat{\Theta})\\
    &\overset{(k)}{=}\mi (\Theta; \Bernoulli (\phi_{\beta}(\innerproduct{\hat{A}}{\Theta}))  \mid \hat{\Theta})\\
    &\overset{(l)}{=}\bE[\mi (\phi_{\beta}(\innerproduct{\hat{A}}{\Theta}); \Bernoulli (\phi_{\beta}(\innerproduct{\hat{A}}{\Theta}))  )\mid \hat{\Theta}=\theta],
\end{align*}
where (i) follows as $\pi_\star$ is a one-to-one mapping; (j) follows from the chain-rule; (k) follows as $\Theta$ and $\hat{\Theta}$ are independent conditioned on the history; %(k) follows as $\Theta$ and $\hat{\Theta}$ are identically distributed conditioned on the history; 
and (l) follows as $\phi_\beta(\innerproduct{a}{\theta})$ is a one-to-one mapping conditioned on $\hat{\Theta}=\theta$. Finally, applying~\Cref{lemma:LowerBoundOnMiForBernoulli} yields the desired result.  

%We note that this result would not be possible to obtain for mutual information of the ``one-step compressed Thompson Sampling'' as equality (k) requires \notes{[to be completed]}. 
\end{proof}

\subsection{Upper bounding the squared expected regret}
\label{subsec:upper_bound_squared_regret}

% Observing that $\innerproduct{A^\star}{\Theta} = ||\Theta||^2_2 = 1 $, we have indeed 
% \begin{align*}
% \phi_\beta(\innerproduct{A^\star}{\Theta})-\phi_\beta(\innerproduct{\hat{A}}{\Theta}) &= \phi_\beta(1)-\phi_\beta(\innerproduct{\hat{A}}{\Theta})= \psi_\beta(1-\innerproduct{\hat{A}}{\Theta}).
%\end{align*}
This part of the analysis is inspired by the proof techniques of~\citet[Proposition~15]{dong_performance_2019}
%\cite{dong_information-theoretic_2018}
and similar to them, the two following lemmata will be of importance for our analysis.

\begin{mylemma}[{\citet[Lemma~16]{dong_performance_2019}}]
\label[lemma]{lemma:exp_d_dimension}
    Let $U, V$ be random vectors in $\mathbb{R}^d$, and let $\tilde{U}, \tilde{V}$ be independent random variables with distributions equal respectively to the marginals of $U, V$, then
    \begin{align*}
        \mathbb{E} \left[ \left( U^\top V \right) \right]^2 \leq d \cdot \mathbb{E} \left[ \left( \tilde{U}^\top \tilde{V} \right)^2 \right].
    \end{align*}
\end{mylemma}

\begin{mylemma}[{\citet[Lemma~18]{dong_performance_2019}}]
\label[lemma]{lemma:ratio_exp_vs_ratio_var}
    Let $f : \mathbb{R}_+ \to \mathbb{R}_+$ be such that $f(0) \geq 0$ and $f(\zeta)/\zeta$ is non-decreasing over $\zeta \geq 0$. Then, for any non-negative random variable $U$, there is
    \begin{align*}
        \frac{\bE[f(U)]^2}{\bE[U]^2} \leq \frac{\bV[f(U)]}{\bV[U]}.
    \end{align*}
\end{mylemma}

Under the assumption that $\cO \subseteq \mathbf{S}_d(0,1)$ with  $\cO \subseteq \cA$, then for each $\theta \in \cO$, there exists an action $a \in \cA$ such that $a = \theta$ and $\innerproduct{a}{\theta} = 1$. This implies that $\phi_\beta(\innerproduct{A^\star}{\Theta}) - \phi_\beta(\innerproduct{\hat{A}}{\Theta}) = \phi_\beta(1) - \phi_\beta(\innerproduct{\hat{A}}{\Theta})$. To simplify notation, we define $\psi_\beta(x) \coloneqq \phi_\beta(1) - \phi_\beta(1 - x)$, which relates the difference between the optimal action $A^\star$ and the sampled action $\hat{A}$ to their corresponding reward differences. Specifically, we have $\psi_\beta(1 - \innerproduct{\hat{A}}{\Theta}) = \psi_\beta(\innerproduct{A^\star - \hat{A}}{\Theta}) = \phi_\beta(\innerproduct{A^\star}{\Theta}) - \phi_\beta(\innerproduct{\hat{A}}{\Theta})$.\\

The function $\psi_\beta(x)$ meets the first two conditions from~\Cref{lemma:ratio_exp_vs_ratio_var} and when it is applied to the difference of inner products $\smash{\innerproduct{A^\star}{\Theta} - \innerproduct{\hat{A}}{\Theta}}$, it maps the interval $[0,2]$ to $[0,1]$, and satisfies $\smash{\psi_\beta(0) = \phi_\beta(1) - \phi_\beta(1 - 0) = 0}$. However, it does not meet the third condition, as $\psi_\beta(x)/x$ initially increases, reaches a maximum between $1$ and $2$, and then decreases (see~\Cref{rem:psi_over_x}
). To address this issue, we introduce a modified function, referred to as the \emph{logistic surrogate}, which serves as the tightest upper bound on $\psi_\beta(x)$ that satisfies the last requirement from~\Cref{lemma:ratio_exp_vs_ratio_var}.  

\begin{mydefinition}[Logistic surrogate]
    \label[definition]{def:logistic_surrogate}
    We construct the \emph{logistic surrogate} function $\varphi_{\beta}(x)$ as the tightest upper bound on $\psi_{\beta}(x)$ such that $\varphi_{\beta}(x)/x$ is non-decreasing over $x \geq 0$.

    Namely, let $\delta_\beta = \arg \max_{x\in[0,2]} \frac{\psi_{\beta}(x)}{x}$, we define $\varphi_{\beta}$ as
    \begin{align*}
        \varphi_{\beta}(x) = \begin{cases}
            \psi_{\beta}(x) &x \in [0,\delta_\beta]\\
            \psi_{\beta}(\delta_\beta) + (x-\delta_\beta)\cdot \psi_{\beta}(\delta_\beta)/\delta_\beta & x \in ~]\delta_\beta,2]
        \end{cases}.
    \end{align*}
\end{mydefinition}

We are now equipped to state and prove an upper bound on the squared expected regret. 
\begin{mylemma}
\label[lemma]{lemma:upper_bound_on_squared_exp_regret}
Let the logistic surrogate be defined as in~\Cref{def:logistic_surrogate}. Then, it holds that 
\begin{align*}
    \bE[ &\Bern{A^\star}{\Theta} - \Bern{\hat{A}}{\Theta}]^2\leq d \cdot \bE\left[ \bV\left[ \varphi_{\beta}\left(1-\innerproduct{\hat{A}}{\Theta}\right) \smash{\mid \hat{\Theta}} \right] \right].
\end{align*}
\end{mylemma}
\begin{proof}
    Integrating over the randomness of the Bernoulli outcome, the squared expected regret can be expressed as $\bE[(\phi_\beta(\innerproduct{A^\star}{\Theta}) - \phi_\beta(\innerproduct{\hat{A}}{\Theta}))]^2 = \bE[\psi_\beta(1 - \innerproduct{\hat{A}}{\Theta})]^2$.
    Since by definition $\varphi_{\beta}(x) \geq \psi_\beta(x)$, we have $\bE[\psi_\beta(1-\innerproduct{\hat{A}}{\Theta})]^2 \leq \bE[\bE[\varphi_\beta(1-\innerproduct{\hat{A}}{\Theta})|\hat{\Theta}]]^2$.
    
    We now apply~\Cref{lemma:ratio_exp_vs_ratio_var} and have that 
    \begin{align*}
        \bE[\bE[\varphi_\beta(1-\innerproduct{\hat{A}}{\Theta})|\hat{\Theta}]]^2 &\leq \bE\Bigg[ \underbrace{\sqrt{\frac{\bV \left[ \varphi_{\beta} \left(1- \innerproduct{\hat{A}}{\Theta} \right) \mid \hat{\Theta} \right]}{\bV \left[ 1-\innerproduct{\hat{A}}{\Theta}  \mid \hat{\Theta}\right]}}}_{\coloneqq U(\hat{\Theta})}  \bE \left[ 1-\innerproduct{\hat{A}}{\Theta}  \mid \hat{\Theta} \right] \Bigg]^2 \\
    &= \bE\left[ U(\hat{\Theta})  \innerproduct{\hat{A}}{\hat{\Theta}}-\innerproduct{\hat{A}}{\Theta}  \right]^2=\bE \left[  \innerproduct{U(\hat{\Theta}) \hat{A}}{\Theta - \hat{\Theta}}  \right]^2_. 
    \end{align*}
    We use~\Cref{lemma:exp_d_dimension} with $U\smash{=}U(\hat{\Theta})\hat{A}$ and $V\smash{=}\Theta\smash{-}\hat{\Theta}$ and rearrange terms to obtain the claimed result: 
    \begin{align*}
     \bE \left[  \innerproduct{ U(\hat{\Theta})\hat{A}}{\Theta - \hat{\Theta}}  \right]^2 &\leq d\cdot \bE \left[ \left(\innerproduct{U(\hat{\Theta})\hat{A}}{\Theta-\tilde{\Theta}} \right)^2 \right]=  d\cdot \bE \left[ U(\hat{\Theta})^2\bE\left[\innerproduct{\hat{A}}{\Theta-\tilde{\Theta}}^2\big|\hat{\Theta}\right]\right]\\
    &=  d\cdot \bE \left[\frac{\bV \left[ \varphi_{\beta} \left(1- \innerproduct{\hat{A}}{\Theta} \right) \mid \hat{\Theta} \right]}{\bV \left[ 1-\innerproduct{\hat{A}}{\Theta}  \mid \hat{\Theta}\right]} \bV \left[\innerproduct{\hat{A}}{\Theta}  \mid \hat{\Theta}\right]\right]\\
    &=  d\cdot \bE \left[\bV \left[ \varphi_{\beta} \left(1- \innerproduct{\hat{A}}{\Theta} \right) \mid \hat{\Theta} \right]\right].
    \end{align*}
\end{proof}

Combining~\Cref{lemma:lower_bound_on_mi} and~\Cref{lemma:upper_bound_on_squared_exp_regret}, we get that the information ration $\Gamma$ is bounded by 
\begin{align*}
    \Gamma \leq d/2 \cdot \frac{\bE\left[\bV \left[ \varphi_{\beta} \left(1- \innerproduct{\hat{A}}{\Theta} \right) \mid \hat{\Theta} \right]\right]}{\bE\left[\bV \left[ \psi_{\beta} \left(1- \innerproduct{\hat{A}}{\Theta} \right) \mid \hat{\Theta} \right]\right]},
\end{align*}
where we use the fact that $\bV\left[ \phi_{\beta}(\innerproduct{\hat{A}}{\Theta}) \mid \hat{\Theta} \right] = \bV \left[ \psi_{\beta} (1- \innerproduct{\hat{A}}{\Theta} ) \mid \hat{\Theta} \right]$ by the definition of $\psi_{\beta}$.
The next part of the proof takes care of controlling the ratio of expected variances over $\varphi_\beta$ and $\psi_\beta$.

\subsection{Bounding the ratio of expected variances over the functions $\varphi_{\beta}$ and $\psi_{\beta}$}
\label{subsec:ratio_exp_variances}

By definition, the function $\psi_\beta$ and its surrogate $\varphi_\beta$ are equal for $x\in[0,\delta_\beta]$ and then diverge linearly at a rate of $\psi_{\beta}(\delta_\beta)/\delta_\beta$. We observe, in~\Cref{rem:psi_over_x}, that $\delta_\beta$ is a decreasing function of $\beta$ and that the slope $\psi_{\beta}(\delta_\beta)/\delta_\beta$ strictly increases with $\beta$.This observation suggests that studying the case $\beta \to \infty$ could provide a general upper bound. Indeed, taking the limit case $\beta\to\infty$, the domain where the two functions differ is maximized, and the rate at which they differ is the largest.\\

We show in~\Cref{lemma:case_infinity_is_upper_bound}, presented in~\Cref{sec:ratio_exp_variances_appendix}, that under some simple preliminary transformations, increasing the value of $\beta$ leads to a larger ratio of expected variances, and therefore, the case $\beta \to \infty$ can serve to derive general upper bounds. Quite satisfyingly, this limit case provides a lot of simplifications. We prove in~\Cref{lemma:case_infinity_is_upper_bound}, that the ratio of expected variance between $\psi_\beta$ and $\varphi_\beta$ is upper bounded by the ratio of expected variance between $\overline{\psi}$ and $\overline{\varphi}$ defined as 
% \begin{align}
%     \label{eq:overline_psi_varphi}
%         &\overline{\psi}(x) = \begin{cases}
%             0 &x \in [0,1],\\
%             1 &x \in ~]1,2],\\
%         \end{cases}
%         &\text{and}
%         &&\overline{\varphi}(x) = \begin{cases}
%             0 &x \in [0,1],\\
%             1+2(x-1) &x \in ~]1,2].\\
%         \end{cases}
%     \end{align}
\begin{align}
    \label{eq:overline_psi}
        \overline{\psi}(x) = \begin{cases}
            0 &x \in [0,1]\\
            1 &x \in ]1,2]\\
        \end{cases}, \quad \quad \textnormal{and}  \quad \quad \overline{\varphi}(x) = \begin{cases}
            0 &x \in [0,1]\\
            1+2(x-1) &x \in ]1,2]\\
        \end{cases}.
    \end{align}
    % and
    % \begin{align}
    % \label{eq:overline_varphi}
    %     \overline{\varphi}(x) = \begin{cases}
    %         0 &x \in [0,1]\\
    %         1+2(x-1) &x \in ]1,2]\\
    %     \end{cases}.
    % \end{align}

\begin{restatable}{mylemma}{UpperBoundOnTheRatioAtInfinity}
\label[lemma]{lemma:upper_bound_on_the_ratio_at_infinity}
Let $\overline{\psi}$ and $\overline{\varphi}$ be defined as in~\cref{eq:overline_psi}.%and~\cref{eq:overline_varphi}. %be defined as in~\cref{eq:overline_psi_varphi}. 
Then, it holds that 
    \begin{align*}
 \frac{\bE \left[\bV \left[ \overline{\varphi} \left(\smash{1- }\innerproduct{\hat{A}}{\Theta} \right) \mid\hat{\Theta} \right]\right]}{\bE\left[\bV \left[ \overline{\psi} \left(\smash{1- }\innerproduct{\hat{A}}{\Theta} \right) \mid\hat{\Theta} \right]\right]} \leq 9.
\end{align*}
\end{restatable}
\paragraph{Sketch of proof}\!\!\!\textit{Analyzing $ \overline{\psi}$, we note that  $\bE [\bV [ \overline{\psi} (\smash{1- }\innerproduct{\hat{A}}{\Theta} ) \mid\hat{\Theta} ]]$ is equal to the expected variance of a Bernoulli random variable 
with probability given by $Q(\hat{A})=\bE[ I(\innerproduct{\hat{A}}{\Theta})]$ where $I(\innerproduct{\hat{A}}{\Theta})\coloneqq \mathds{1}_{\{\innerproduct{\hat{A}}{\Theta} < 0\}}$. The expected variance can then be written as $\bE[ I(\innerproduct{\hat{A}}{\Theta})^2]- \bE[\bE[ I(\innerproduct{\hat{A}}{\Theta})]^2]$ where in the second term, the outer expectation is on $\hat{A}$, and the inner one is on $\Theta$.  After rearranging terms, we can write $\bE [\bV [ \overline{\varphi} (\smash{1- }\innerproduct{\hat{A}}{\Theta}) \mid\hat{\Theta} ]]$ as $\bE[I(\innerproduct{\hat{A}}{\Theta})(1-2\innerproduct{\hat{A}}{\Theta} )^2] -\bE[\bE[I(\innerproduct{\hat{A}}{\Theta})(1-2\innerproduct{\hat{A}}{\Theta} )]^2]$ where again for the second term, the outer expectation is on $\hat{A}$, and the inner one is on $\Theta$. Taking the supremum over the possible values of $(1-2\innerproduct{a}{\theta}) \in [-1,3]$ concludes the proof. }

% \section{Numerical simulations}
% \label{sec:numerical_simulations}
% \input{appendix_numerical_experiment}

\section{Conclusion and Future Work}
\label{sec:conclusion}
In this work, we analyzed the performance of the Thompson Sampling algorithm for logistic bandit problems. %focusing on settings where both the action and parameter spaces lie within the $d$-dimensional unit ball. 
Following the information-theoretic framework from~\cite{russo_information-theoretic_2015}, we study the information ratio, a key statistic that captures the trade-off between exploration and exploitation in logistic bandits. Our main result establishes that the information ratio of Thompson Sampling for logistic bandits can be bounded using only the dimension of the problem, $d$, and $\alpha$, a minimax alignment constant between the action and parameter spaces. Importantly, this bound is independent of the logistic function's slope parameter, $\beta$.\\

Building on this result, we derive a regret bound of $O(d/\alpha\sqrt{T \log(\beta T/d)})$, which scales only logarithmically with $\beta$, representing a significant improvement over prior works. To the best of our knowledge, this is the first regret bound for logistic bandits that achieves logarithmic dependence on $\beta$ while remaining independent of the action set's cardinality. 
%Importantly, our results do not depend on the fragility dimension $\eta$, unlike those of~\citet{dong_performance_2019}. This distinction is significant because, except in cases where $\smash{\alpha = 1}$, the fragility dimension can grow exponentially with the dimension $d$. 
Finally, we presented specific settings where the dependence on $\alpha$ can be controlled. For instance, when the action space fully encompasses the parameter space, the regret of Thompson Sampling scales as $\tilde{O}(d \sqrt{T})$.\\

A promising direction for future work is to extend our analysis to the broader class of generalized linear bandits. The properties of the logistic function that we leverage in our analysis could be shared by other class of link function and be used to derive regret bounds using a similar analysis of the information ratio as we performed in~\Cref{sec:analysis}.\\

Another interesting research direction is to use the result in this paper to derive regret bounds for logistic bandits in the frequentist setting. A promising way is to apply our information theoretic analysis to the optimistic information directed sampling algorithm introduced by~\citet{neu_optimistic_2024}. We believe that this approach could lead to deriving new and improved frequentist bounds for logistic bandits. %avenue is exploring whether our approach can be applied to the \emph{optimistic information directed sampling} algorithm introduced by~\citep{neu_optimistic_2024}, with the goal of deriving frequentist regret bounds for logistic bandits that scale logarithmically with the parameter $\beta$. Extending our analysis to the frequentist setting would represent a significant advancement in the field.

\vspace{-0.5em}
\bibliographystyle{IEEEtranN}
\bibliography{references.bib}
\appendix

\crefalias{section}{appendix}
\clearpage

\newpage
\appendix
\section*{Appendix}
The appendix is organized as follows:
\begin{itemize}
    \item \Cref{sec:proof_of_useful_lemmata} introduces lemmata that will be useful later for our proofs;
    \item \Cref{sec:ratio_exp_variances_appendix} formalizes the proof for controlling the ratio of expected variances between the functions $\varphi_\beta$ and $\psi_\beta$;
    \item \Cref{sec:extension} extends our information ratio analysis to general action and parameter spaces;
    \item \Cref{sec:numerical_simulations} illustrates the improvement of our bounds compared to previous regret guarantees through numerical experiments;
    \item \Cref{sec:adaptation} illustrates how applying directly linear bandits bounds to the logistic bandit setting leads to regret bounds scaling exponentially with $\beta$;
    \item \Cref{sec:arguments_gaps_in_dong} elaborates on the gaps in the previous literature mentioned in~\Cref{sec:introduction};
    \item \Cref{sec:constructing_the_one_to_one_mapping} rigorously explains the construction of the mapping $\pi_\star$.
\end{itemize}

\section{Useful lemmata}
\label{sec:proof_of_useful_lemmata}
\LowerBoundOnMiForBernoulli*

\begin{proof}
Using~\citet[Theorem 3.4.d]{yury_polyanskiy_information_2022}, we decompose the mutual information between $U$ and $\Bernoulli(U)$ as
    \begin{align*}
        \mi(U;\Bernoulli(U)) = h(\Bernoulli(U))- h(\Bernoulli(U)|U).%\label{eq:mi_U_Bern_U}.
    \end{align*}
    Following~\citet[Example 2.2]{duchi_lecture_2016} notation, we define $h_2(p) \coloneqq -p\log(p) - (1-p)\log(1-p)$ for $p\in [0,1]$ and rewrite the mutual information as
    \begin{align}
        \mi(U;\Bernoulli(U)) &= h_2(\bE[U]) - \bE[h_2(U)].
        \label{eq:mi_as_diff_ent}
    \end{align}
    From a Taylor expansion of $h_2(x)$ we have that 
    \begin{align*}
        h_2(x) = h_2(p) + (x-p) h_2'(p) + \frac{1}{2} (x-p)^2 h_2''(\xi), 
    \end{align*}
    for some $\xi \in (0,1)$ as $ h_2''$ is continuous on the interval $[0,1]$. 
    We compute the second derivative of $h_2$ and get $h_2''(\xi) = - \frac{1}{\xi(1-\xi)}$ for $\xi \in (0,1)$. This function is concave and maximal at $\xi=1/2$, where it takes the value $h_2''(1/2)=-4$. 
    We then have that for all $x \in [0,1]$ and all $p \in [0,1]$,   
    \begin{align*}
        h_2(x) \leq h_2(p) + (x-p) h_2'(p) -2(x-p)^2.
    \end{align*}
    Using this fact for $x=U$ and $p=\bE[U]$, we have that 
    \begin{align*}
    h_2(U)\leq h_2(\bE[U]) +(U-\bE[U]) h_2'(\bE[U]) - 2(U-\bE[U])^2.
    \end{align*}
    Applying the last inequality to the second term in~\cref{eq:mi_as_diff_ent},  it comes that
    \begin{align*}
        \mi(U;\Bernoulli(U)) &\geq \bE\left[h_2(\bE[U])\smash{ - h_2(\bE[U])}\smash{-(}\smash{U-}\bE[U])h_2'(\bE[U])\smash{+ }2\smash{(U-}\bE[U])^2 \right].
    \end{align*}
    Finally,  simplifying terms and taking the expectation gives the desired result.
  \end{proof}

 The two following lemmata are be particularly useful to control the covering number in Euclidean balls and spheres. 
\begin{mylemma}[{\citet[Lemma 5.13]{van_handel_probability_2016}}]
\label[lemma]{lemma:van_handel}
Let $\mathbf{B}_d(0,1)$ denote the $d$-dimensional closed Euclidean unit ball. We have $|\cN(\mathbf{B}_d(0,1),||\cdot||_2,\varepsilon)| =1 $ for $\varepsilon\geq 1$ and for $0<\varepsilon<1$, we have
\begin{align*}
    &\bigg(\frac{1}{\varepsilon}\bigg)^d\leq|\cN(\mathbf{B}_d(0,1),||\cdot||_2,\varepsilon)|\leq  \bigg(1+\frac{2}{\varepsilon}\bigg)^d .
\end{align*}

\end{mylemma}

\begin{mylemma}[{\citet[Corollary 27.4]{yury_polyanskiy_information_2022}}]
\label[lemma]{lemma:polyanskiy}
    Let $\mathbf{S}_d(0,1)$ denote the $d$-dimensional Euclidean unit sphere. We have $|\cN(\mathbf{S}_d(0,1),||\cdot||_2,\varepsilon) =1 $ for $\varepsilon\geq 1$ and for $0<\varepsilon<1$, we have
    \begin{align*}
    &\bigg(\frac{1}{2\varepsilon}\bigg)^{d-1}\leq|\cN(\mathbf{S}_d(0,1),||\cdot||_2,\varepsilon)|\leq  2d\bigg(1+\frac{1}{\varepsilon}\bigg)^{d-1} .
\end{align*}
\end{mylemma}

% \section{Proof of~\Cref{thm:quantized_bound}}
% \label{sec:proof_of_quantized_bound}
% \input{appendix_general_thm}

\section{Bounding the ratio of expected variances over the functions $\varphi_\beta$ and $\psi_\beta$}
\label{sec:ratio_exp_variances_appendix}

\UpperBoundOnTheRatioAtInfinity*
\begin{proof}
We start by analyzing $\bE \left[\bV \left[ \overline{\psi} \left(\smash{1- }\innerproduct{\hat{A}}{\Theta} \right) \mid\hat{\Theta} \right]\right]$. We note that $\overline{\psi} \left(\smash{1- }\innerproduct{\hat{A}}{\Theta} \right)$ is equal to $1$ if $\innerproduct{\hat{A}}{\Theta}<0$ and is equal to $0$ otherwise.  To distinguish those two cases, we introduce the notation $I(\innerproduct{\hat{A}}{\Theta})\coloneqq \mathds{1}_{\{\innerproduct{\hat{A}}{\Theta} < 0\}}$.  We observe that $\bE \left[\bV \left[ \overline{\psi} \left(\smash{1- }\innerproduct{\hat{A}}{\Theta} \right) \mid\hat{\Theta} \right]\right]$ is equal to the expected variance of a Bernoulli random variable with probability given by $Q(\hat{A})\coloneqq\bE[ I(\innerproduct{\hat{A}}{\Theta})]$ and can therefore be written as $\bE \left[\bV \left[ \overline{\psi} \left(\smash{1- }\innerproduct{\hat{A}}{\Theta} \right) \mid\hat{\Theta} \right]\right] = \bE[Q(\hat{A})(1-Q(\hat{A}))]$.\\
    % \begin{align*}
    %     \bE \left[\bV \left[ \overline{\psi} \left(\smash{1- }\innerproduct{\hat{A}}{\Theta} \right) \mid\hat{\Theta} \right]\right] = \bE[Q(\hat{A})(1-Q(\hat{A}))].
    % \end{align*}
    The last part of the proof concerns $\bE \left[\bV \left[ \overline{\varphi} \left(\smash{1- }\innerproduct{\hat{A}}{\Theta} \right) \mid\hat{\Theta} \right]\right]$.  Similarly,  we can distinguish between two cases: either $\innerproduct{\hat{A}}{\Theta}\smash{>0}$ and $\overline{\varphi}(\smash{1-}\innerproduct{\hat{A}}{\Theta})=0$,  or  $\innerproduct{\hat{A}}{\Theta}\smash{<0}$ and $\overline{\varphi}(\smash{1-}\innerproduct{\hat{A}}{\Theta})=1-2\innerproduct{\hat{A}}{\Theta}$. Introducing the notation $G(\hat{A})\coloneqq\bE[ I(\innerproduct{\hat{A}}{\Theta})\innerproduct{\hat{A}}{\Theta}]$, we can write  
    \begin{align*}
    \bE \left[\bV \left[ \overline{\varphi} \left(\smash{1- }\innerproduct{\hat{A}}{\Theta} \right) \mid\hat{\Theta} \right]\right]&= \bE \left[ \bE \left[ \left(\overline{\varphi} \left(\smash{1- }\innerproduct{\hat{A}}{\Theta} \right)-\bE\left[\overline{\varphi} \left(\smash{1- }\innerproduct{\hat{A}}{\Theta} \right) \mid\hat{\Theta}\right] \right)^2  \mid\hat{\Theta}\right]\right]\\
    &= \bE\left[I(\innerproduct{\hat{A}}{\Theta})\left(1-2\innerproduct{\hat{A}}{\Theta}-\left(Q(\hat{A})+2 G(\hat{A})\right) \right)^2\right]\\
    &\quad+ \bE\left[\left(1-I(\innerproduct{\hat{A}}{\Theta})\right)\left(0-\left(Q(\hat{A})+2 G(\hat{A})\right) \right)^2\right].
   \end{align*}

   Distributing the square and simplifying terms, we obtain 
   \begin{align*}
    &\bE\left[I(\innerproduct{\hat{A}}{\Theta})\left(1-2\innerproduct{\hat{A}}{\Theta} \right)^2\right] -2\bE\left[I(\innerproduct{\hat{A}}{\Theta})\left(1-2\innerproduct{\hat{A}}{\Theta} \right)\left(Q(\hat{A})+2 G(\hat{A})\right)\right]\\
    &+ \bE\left[I(\innerproduct{\hat{A}}{\Theta})\left(Q(\hat{A})+2 G(\hat{A})\right)^2\right] + \bE\left[\left(1-I(\innerproduct{\hat{A}}{\Theta})\right)\left(Q(\hat{A})+2 G(\hat{A})\right)^2\right]\\
    &=\bE\left[I(\innerproduct{\hat{A}}{\Theta})\left(\smash{1-2}\innerproduct{\hat{A}}{\Theta} \right)^2\right]\smash{ -}\bE\left[\left(Q(\hat{A})+2 G(\hat{A})\right)^2\right].
    \end{align*}
    To get to the last part of the proof, we rewrite explicitly $Q(\hat{A})+2 G(\hat{A})$ as $\bE\left[I(\innerproduct{\hat{A}}{\Theta})\left(1-2\innerproduct{\hat{A}}{\Theta} \right)\right]$ and optimize over the values of $(1-2\innerproduct{\hat{A}}{\Theta})$. It then comes 
    \begin{align*}
    \bE&\left[I(\innerproduct{\hat{A}}{\Theta})\left(1-2\innerproduct{\hat{A}}{\Theta} \right)^2\right] -\bE\left[\bE\left[I(\innerproduct{\hat{A}}{\Theta})\left(1-2\innerproduct{\hat{A}}{\Theta} \right)\right]^2\right]\\
     &\leq \sup_{\zeta \in [-1,3]} \bE\left[I(\innerproduct{\hat{A}}{\Theta})\zeta^2\right] - \bE\left[\bE\left[I(\innerproduct{\hat{A}}{\Theta})\zeta \right]^2\right] = 9 \cdot \bE[Q(\hat{A})(1-Q(\hat{A}))],
    \end{align*}
    which concludes the proof. 
\end{proof}

\begin{mylemma}
\label[lemma]{lemma:case_infinity_is_upper_bound}
Let $\psi_\beta(x) = \phi_\beta(1) - \phi(1-x)$ and the logistic surrogate $\varphi_\beta$ as in~\Cref{def:logistic_surrogate} and let  $\overline{\psi}$ and $\overline{\varphi}$ be defined as  in~\cref{eq:overline_psi}. %and~\cref{eq:overline_varphi}.
Then, for all $\beta>0$, it holds that
    \begin{align*}
\frac{\bE \left[\bV \left[ \varphi_{\beta} \left(1-\innerproduct{\hat{A}}{\Theta} \right)\smash{ \mid }\Theta \right]\right]}{\bE \left[\bV \left[ \psi_{\beta} \left(1-\innerproduct{\hat{A}}{\Theta} \right) \mid\hat{\Theta} \right]\right]} \leq \frac{\bE \left[\bV \left[ \overline{\varphi} \left(1-\innerproduct{\hat{A}}{\Theta} \right) \mid\hat{\Theta} \right]\right]}{\bE \left[\bV \left[ \overline{\psi} \left(1-\innerproduct{\hat{A}}{\Theta} \right) \mid\hat{\Theta} \right]\right]}.
\end{align*}
\end{mylemma}
\begin{proof}
Beginning with the ratio of expected variances between $\varphi_{\beta}$ and $\psi_{\beta}$, we will apply a series of transformations to the functions $\varphi_{\beta}$ and $\psi_{\beta}$, ultimately yielding the functions $\overline{\varphi}$ and $\overline{\psi}$. These transformations are chosen to ensure they can only increase the ratio of expected variances.

By definition, the function $\psi_\beta$ and its surrogate $\varphi_\beta$ are identical for $x \in [0, \delta_\beta)$ and then diverge linearly at a rate of $\psi_{\beta}(\delta_\beta)/\delta_\beta$ on the interval $x \in [\delta_\beta, 2]$. We illustrate this on~\Cref{fig: psi_beta}. Focusing on the domain where the two functions coincide, we observe that the transformation $f(x) = \max(x, \psi_\beta(1))$ reduces the expected variance for both $\psi_\beta$ and $\varphi_\beta$. However, since $\psi_\beta(x)$ is less than or equal to $\varphi_\beta(x)$ for all $x\in [0,2]$, and both functions exceed $\psi_\beta(1)$ on the interval $[1, 2]$, the transformation $f$ proportionally reduces the expected variance of $\psi_\beta$ more than that of $\varphi_\beta$. As a result, the transformation increases the ratio of expected variances between the two functions. As $\psi_\beta$ and $\varphi_\beta$ are strictly increasing functions, the resulting functions, illustrated on~\Cref{fig:f_psi_beta}, can be written as
\begin{align*}
   f(\psi_\beta(x)) = \begin{cases}
       \psi_\beta(1) & x \in [0,1]\\
       \psi_\beta(x) & x \in ~]1,2]
   \end{cases}, \quad \quad \textnormal{and} \quad \quad f(\varphi_\beta(x)) = \begin{cases}
       \psi_\beta(1) & x \in [0,1]\\
       \varphi_\beta(x) & x \in ~]1,2]
   \end{cases}.
\end{align*}
% and 
% \begin{align*}
%    f(\varphi_\beta(x)) = \begin{cases}
%        \psi_\beta(1) & x \in [0,1]\\
%        \varphi_\beta(x) & x \in ~]1,2]
%    \end{cases}.
% \end{align*}

The second transformation we apply concerns only the function $f(\psi_\beta(x))$. We crop all the values larger than $\psi(\delta_\beta)$ by applying the transformation $g(x)=\min(x,\psi(\delta_\beta))$. As $f(\psi_\beta(x))$ is an increasing function, the function $g(f(\psi_\beta(x)))$, illustrated on~\Cref{fig:g_f_psi_beta}, can be written as 
\begin{align*}
   g(f(\psi_\beta(x))) = \begin{cases}
       \psi_\beta(1) & x \in [0,1]\\
       \psi_\beta(x) & x \in ~]1,\delta_\beta]\\
       \psi_\beta(\delta_\beta) & x \in ~]\delta_\beta,2]
   \end{cases}.
\end{align*}
The transformation $g$ reduces the variance of the function $f(\psi_\beta(x))$ as it both decreases the values of $f(\psi_\beta(x))$ and the derivative of $f(\psi_\beta(x))$ for all $x\in]\delta_\beta,2]$.

The third transformation we apply is increasing the value of $\beta$. As $\beta$ increases, the derivative of $f(\varphi_\beta(x))$ increases everywhere, 
\begin{align*}
    \frac{d}{d x} f(\varphi_\beta(x)) = \begin{cases}
        0 & x \in [0,1]\\
       \frac{\beta \exp(-\beta(1-x))}{\left(1+\exp(-\beta(1-x)) \right)^2} & x \in ~]1,\delta_\beta]\\
       \psi_\beta(\delta_\beta)/\delta_\beta & x \in ~]\delta_\beta,2]
    \end{cases}.
\end{align*}
and the expected variance of $f(\varphi_\beta)$ increases.  
Regarding $g(f(\psi_\beta(x)))$, we can show that that for all $x\in [0,2]$, the ratio $f(\varphi_\beta(x))/g(f(\psi_\beta(x)))$
increases with $\beta$. Indeed, this ratio is equal to $1$ for all $x\in [0,\delta_\beta]$ and increases for  all $x\in ~]\delta_\beta,2]$ as 
\begin{align*}
    \frac{f(\varphi_\beta(x))}{g(f(\psi_\beta(x)))} = \frac{\varphi_\beta(\delta_\beta)+\varphi_\beta(\delta_\beta)/\delta_\beta \cdot(x-\delta_\beta)}{\varphi_\beta(\delta_\beta)} = 1 + \frac{(x-\delta_\beta)}{\delta_\beta},
\end{align*}
and as $\delta_\beta$ is a decreasing function of $\beta$ (see~\Cref{rem:psi_over_x}), the ratio $(x-\delta_\beta)/\delta_\beta$ is a increasing function of $\beta$ for all $x\in ~]\delta_\beta,2]$. This fact ensures that the expected variance of $g(f(\psi_\beta(x)))$ cannot increase proportionally more than the expected variance of $f(\varphi_\beta(x))$. We can therefore study the ratio of expected variances between $f(\varphi_\infty)$ and $g(f(\psi_\infty))$.

The last operation we apply is a convenient shifting and scaling. We define $h(x) = (x-g(f(\psi_\beta(1))))/(g(f(\psi_\infty(2)))-g(f(\psi_\beta(1))))$ and apply it on both $g(f(\psi_\infty))$ and $f(\varphi_\infty)$ these operations do not affect the ratio of expected variances. The resulting functions are illustrated on~\Cref{fig:psi_infinity}.

To express the resulting functions, we have to analyze the function $\psi_\beta(x)$ for $\beta$ tending to infinity for values $x\in]1,2]$. 

We recall that $\psi_\beta(x)= \phi_\beta(1)-\phi_\beta(1-x)$ and can equivalently be written as
\begin{align*}
        \psi_\beta(x) = \frac{1}{1+\exp(-\beta)}-\frac{1}{1+\exp(-\beta (1-x))}.
\end{align*} 

We have to distinguish between three cases for $(x-1)$: negative, zero, or positive. For values of $x\in]1,2]$, we have that$(1-x)<0$ and that $\lim_{\beta \to  \infty} \psi_\beta(x) = 1$, if $x=1$, we have that $\lim_{\beta \to  \infty} \psi_\beta(x) = 1/2$ and for values of $x\in[0,1[$, we have that$(1-x)>0$ and that $\lim_{\beta \to  \infty} \psi_\beta(x) = 0$. We can then write  
    \begin{align*}
        \psi_{\infty}(x) = \begin{cases}
            0 &x\in[0,1[\\
            1/2 &x=1\\
            1 &x \in ~]1,2]
        \end{cases}.
    \end{align*}
    We can now construct the corresponding $\varphi_{\infty}(x)$. We note that $\frac{\psi_{\infty}(x)}{x}$ is maximized when taking the limit to $x=1^+$ from the right: $\lim_{x\to 1^+}\frac{\psi_{\infty}(x)}{x} = 1$. It comes that $\varphi_{\infty}(x)$ can be written as 
    \begin{align*}
    \varphi_{\infty}(x) = \begin{cases}
        0 &x\in[0,1[\\
        1/2 &x=1\\
        1+(x-1) &x \in ~]1,2]
    \end{cases}.
    \end{align*}

We denote the resulting functions $h(g(f(\psi_{\infty}(x))))$ and $h(f(\varphi_\infty(x)))$ respectively as $\overline{\psi}$ and $\overline{\varphi}$. We note that they can be written quite simply as 
    \begin{align*}
        \overline{\psi}(x) = \begin{cases}
            0 &x \in [0,1]\\
            1 &x \in ~]1,2]\\
        \end{cases}, \quad \quad \textnormal{and} \quad \quad \begin{cases}
            0 &x \in [0,1]\\
            1+2(x-1) &x \in ~]1,2]\\
        \end{cases}.
    \end{align*}
    % and
    % \begin{align*}
    %     \overline{\varphi}(x) = \begin{cases}
    %         0 &x \in [0,1]\\
    %         1+2(x-1) &x \in ~]1,2]\\
    %     \end{cases}.
    % \end{align*}
\end{proof}

\begin{myremark}
\label[remark]{rem:psi_over_x}
%Remark about  $\psi_\beta(x)/x$ that increases initially, reaching a maximum between $1$ and $2$ before decreasing.
 We illustrate the function $\psi_\beta(x)/x$ on~\Cref{fig: function_psi_beta_over_x}
and the behavior of $\delta_\beta$ and $\psi_\beta(\delta_\beta)/\delta_\beta$ for increasing values of $\beta$ on~\Cref{fig: delta_beta}.
The derivative of the function $\psi_\beta(x)/x$ is given by 
    \begin{align*}
        \frac{d}{dx}\left(\frac{\psi_\beta(x)}{x} \right) = \frac{1}{x}\left(\frac{d}{dx}\psi_\beta(x) - \frac{\psi_\beta(x)}{x}\right).
    \end{align*}
We note that it is equal to zero for values of $x\in]0,2]$ such that $\frac{d}{dx}\psi_\beta(x) = \frac{\psi_\beta(x)}{x}$. By definition of $\delta_\beta$, we have $\frac{d}{dx}\psi_\beta(\delta_\beta) = \frac{\psi_\beta(\delta_\beta)}{\delta_\beta}$.

\begin{figure}[ht]
\centering
\begin{minipage}{.465\textwidth}
  \centering
  \resizebox{\linewidth}{!}
    {
    \includegraphics[]{figures/function_psi_beta_over_x.tikz}
    }
  \captionof{figure}{Illustration of the function $\psi_\beta(x)/x$ for different values of $x$. The maximum of the function is attained for $x=\delta_\beta$.}
  \label{fig: function_psi_beta_over_x}
\end{minipage}%
\hspace{2em}
\begin{minipage}{.465\textwidth}
  \centering
  \resizebox{\linewidth}{!}
    {
    \includegraphics[]{figures/delta_beta.tikz}
    }
  \captionof{figure}{Illustration of $\delta_\beta$ and $\psi(\delta_\beta)/\delta_\beta$ as functions of $\beta$. One can observe that $\delta_\beta$ decreases with $\beta$ while $\psi(\delta_\beta)/\delta_\beta$ increases. }
  \label{fig: delta_beta}
\end{minipage}%
\end{figure}

\end{myremark}

\begin{figure}[ht]
\centering
\begin{minipage}{.465\textwidth}
  \centering
  \resizebox{\linewidth}{!}
    {
    \includegraphics[]{figures/psi_beta.tikz}
    }
  \captionof{figure}{Illustration of the function $\psi_\beta$ (in solid line) and the function $\varphi_\beta$ (in dotted line) for different values of $\beta$.}
  \label{fig: psi_beta}
\end{minipage}
\hspace{2em}
\begin{minipage}{.465\textwidth}
  \centering
  \resizebox{\linewidth}{!}
    {
    \includegraphics[]{figures/f_psi_beta.tikz}
    }
  \captionof{figure}{Illustration of the function $f(\psi_\beta)$ (in solid line) and the function $f(\varphi_\beta)$ (in dotted line) for different values of $\beta$.}
  \label{fig:f_psi_beta}
\end{minipage}%
\end{figure}

\begin{figure}[ht]
\centering
\begin{minipage}{.465\textwidth}
  \centering
  \resizebox{\linewidth}{!}
    {
    \includegraphics[]{figures/g_f_psi_beta.tikz}
    }
  \caption{Illustration of the function $g(f(\psi_\beta))$ (in solid line) and the function $f(\varphi_\beta)$ (in dotted line) for different values of $\beta$.}
  \label{fig:g_f_psi_beta}
\end{minipage}
\hspace{2em}
\begin{minipage}{.465\textwidth}
  \centering
  \resizebox{\linewidth}{!}
    {
    \includegraphics[]{figures/psi_infinity.tikz}
    }
  \caption{Illustration of the function $\overline{\psi}$ (in blue) and the function $\overline{\varphi}$ (in orange).}
  \label{fig:psi_infinity}
\end{minipage}%
\end{figure}

\section{Extension to general spaces}
\label{sec:extension}
The extend the proof technique of~\Cref{subsec:upper_bound_squared_regret} and~\Cref{subsec:ratio_exp_variances}, we first need to introduce the \emph{alignment function} $\alpha(\theta) \coloneqq \innerproduct{\pi_\star(\theta)}{\theta}$. We can define the \emph{extended logistic function} $\psi_\beta(x,\theta) \coloneqq \phi_\beta\left(\alpha(\theta)\right)-\phi_\beta\left(\alpha(\theta)-x\right)$ and note that
\begin{align*}
    \psi_\beta(\alpha(\hat{\Theta})-\innerproduct{\hat{A}}{\Theta},\hat{\Theta}) = \phi_\beta(\alpha(\hat{\Theta})) - \phi_\beta(\innerproduct{\hat{A}}{\Theta}) = \phi_\beta(\innerproduct{\hat{A}}{\hat{\Theta}}) - \phi_\beta(\innerproduct{\hat{A}}{\Theta}).
\end{align*}
Integrating the randomness of the Bernoulli process, we can write the expected regret using $\smash{\psi_\beta(x,\theta)}$: 
\begin{align*}
   \bE[\psi_\beta(\alpha(\hat{\Theta})-\innerproduct{\hat{A}}{\Theta},\hat{\Theta})]=\bE[ \phi_\beta(\innerproduct{\hat{A}}{\hat{\Theta}})-\phi_\beta(\innerproduct{\hat{A}}{\Theta})]= \bE[ \phi_\beta(\innerproduct{A^\star}{\Theta})-\phi_\beta(\innerproduct{\hat{A}}{\Theta})],
\end{align*}
where we used the fact that the pair $(A^\star,\Theta)$ and $(\hat{A},\hat{\Theta})$ are identically distributed. 

Similarly to the proof in~\Cref{subsec:upper_bound_squared_regret}, we construct a function $\varphi_\beta(x,\theta)$ as the tightest upper bound on $\psi_\beta(x,\theta)$ that satisfies the requirements of~\Cref{lemma:ratio_exp_vs_ratio_var}.

\begin{mydefinition}[Extended logistic surrogate]
    \label[definition]{def:logistic_surrogate_extended}
    We construct the \emph{extended logistic surrogate} function $\varphi_{\beta}(x,\theta)$ as the tightest upper bound on $\psi_{\beta}(x,\theta)$ such that $\varphi_{\beta}(x,\theta)/x$ is non-decreasing over $x \geq 0$ for all $\theta \in \cO$.

    Namely, let $\delta_\beta(\theta) = \arg \max_{x\in[0,1+\alpha(\theta)]} \frac{\psi_{\beta}(x,\theta)}{x}$, we define the function $\varphi_{\beta}(x,\theta)$ as
    \begin{align*}
        \varphi_{\beta}(x,\theta) = \begin{cases}
            \psi_{\beta}(x,\theta) &x \in [0,\delta_\beta(\theta)]\\
            \psi_{\beta}(\delta_\beta(\theta),\theta) + (x-\delta_\beta(\theta))\cdot \psi_{\beta}(\delta_\beta(\theta),\theta)/\delta_\beta(\theta) & x \in ~]\delta_\beta(\theta),1+\alpha(\theta)]
        \end{cases}.
    \end{align*}
\end{mydefinition}
We are now equipped to extend~\Cref{lemma:upper_bound_on_squared_exp_regret} to general action and parameter spaces.  
\begin{mylemma}
\label[lemma]{lemma:upper_bound_on_squared_exp_regret_extended}
Let the extended logistic surrogate be defined as in~\Cref{def:logistic_surrogate_extended}. Then, it holds that 
\begin{align*}
    \bE&[ \Bern{\hat{A}}{\Theta} - \Bern{\hat{A}}{\Theta}]^2\leq d \cdot \bE\left[ \bV\left[ \varphi_{\beta}(\alpha(\hat{\Theta})-\innerproduct{\hat{A}}{\Theta},\hat{\Theta}) \smash{\mid \hat{\Theta}} \right] \right].
\end{align*}
\end{mylemma}
\begin{proof}
    The proof follows closely the technique used to prove~\Cref{lemma:upper_bound_on_squared_exp_regret}. We note that conditioned on $\hat{\Theta}=\theta$, the extended logistic surrogate is a mapping from $[0,1+\delta_\beta(\theta)]$ to $[0,1]$, that $\varphi(0,\theta) = \phi_\beta(\alpha(\theta)) - \phi_\beta(\alpha(\theta)) = 0 $ and fulfills the assumptions of~\Cref{lemma:ratio_exp_vs_ratio_var}.
\end{proof}

Noting that $\bV\left[ \phi_{\beta}(\innerproduct{\hat{A}}{\Theta}) \mid \hat{\Theta} \right] = \bV \left[ \psi_{\beta} (\alpha(\hat{\Theta})-\innerproduct{\hat{A}}{\Theta},\hat{\Theta}) \mid \hat{\Theta} \right]$ and using~\Cref{lemma:lower_bound_on_mi}, we have that the information ration $\Gamma$ can be bounded by 
\begin{align*}
    \Gamma \leq d/2 \cdot \frac{\bE\left[\bV \left[ \varphi_{\beta} \left(\alpha(\hat{\Theta})- \innerproduct{\hat{A}}{\Theta},\hat{\Theta} \right) \mid \hat{\Theta} \right]\right]}{\bE\left[\bV \left[ \psi_{\beta} \left(\alpha(\hat{\Theta})- \innerproduct{\hat{A}}{\Theta},\hat{\Theta} \right) \mid \hat{\Theta} \right]\right]}.
\end{align*}
Similarly to the analysis for the $\cO \subseteq \cA$, we can derive an upper bound by studying the case $\beta \to \infty$ after applying the same preliminary transformations, $f$, $g$, and $h$ as used in~\Cref{lemma:case_infinity_is_upper_bound} on the functions $\varphi_{\beta}(x,\theta)$ and $\psi_{\beta}(x,\theta)$. To express the resulting functions $h(g(f(\psi_{\infty}(x))))$ and $h(f(\varphi_\infty(x)))$, we need to study the extended logistic function $\psi_\beta(x,\theta)$ and the corresponding extended logistic surrogate $\varphi_\beta(x,\theta)$ for $\beta$ tending to infinity. 

Starting with $\psi_\beta(x,\theta)$, we recall that $\psi_\beta(x,\theta) = \phi_\beta(\alpha(\theta)) - \phi_\beta(\alpha(\theta)-x)$ can be written as
\begin{align*}
    \psi_\beta(x,\theta) = \frac{1}{1+\exp(-\beta\alpha(\theta))}-\frac{1}{1+\exp(-\beta (\alpha(\theta)-x))}.
\end{align*}
Again, we can distinguish between three cases for $(\alpha(\theta)-x))$: negative, zero, or positive. For values of $x\in]\alpha(\theta),1+\alpha(\theta)]$, we have that$(\alpha(\theta)-x)<0$ and that $\lim_{\beta \to  \infty} \psi_\beta(x,\theta) = 1$, if $x=\alpha(\theta)$, we have that $\lim_{\beta \to  \infty} \psi_\beta(x,\theta) = 1/2$ and for values of $x\in[0,\alpha(\theta)[$, we have that$(\alpha(\theta)-x)>0$ and that $\lim_{\beta \to  \infty} \psi_\beta(x,\theta) = 0$. We can then write  
    \begin{align*}
        \psi_{\infty}(x,\theta) = \begin{cases}
            0 &x\in[0,\alpha(\theta)[\\
            1/2 &x=\alpha(\theta)\\
            1 &x \in ~]\alpha(\theta),1+\alpha(\theta)]
        \end{cases}.
    \end{align*}
    We continue and construct the corresponding $\varphi_\infty(x),\theta$.  By definition, we have that $\alpha(\theta)\leq 1$ and we note that $\frac{\psi_{\infty}(x,\theta)}{x}$ is maximized when taking the limit to $x=\alpha(\theta)^+$ from the right: $\lim_{x\to \alpha(\theta)^+}\frac{\psi_{\infty}(x)}{x} = \frac{1}{\alpha(\theta)}$. It comes that $\varphi_\infty(x)$ can be written as 
    \begin{align*}
    \varphi_\infty(x) = \begin{cases}
        0 &x\in[0,\alpha(\theta)[\\
        1/2 &x=\alpha(\theta)\\
        1+\frac{x-\alpha(\theta)}{\alpha(\theta)} &x \in ~]\alpha(\theta),1+\alpha(\theta)]
    \end{cases}.
    \end{align*}
    We can now construct the functions $\overline{\psi}(x,\theta)\coloneqq h(g(f(\psi_{\infty}(x,\theta))))$ and $\overline{\varphi}(x,\theta) \coloneqq h(f(\varphi_\infty(x,\theta)))$ for $f$, $g$, and $h$ defined as in~\Cref{lemma:case_infinity_is_upper_bound}. We note that the resulting functions can be written as 
    \begin{align}
    \label{eq:overline_psi_extended}
        \overline{\psi}(x,\theta) = \begin{cases}
            0 &x \in [0,\alpha(\theta)]\\
            1 &x \in ~]\alpha(\theta),1+\alpha(\theta)]\\
        \end{cases},  \quad \textnormal{and}  \quad \overline{\varphi}(x,\theta) = \begin{cases}
            0 &x \in [0,\alpha(\theta)]\\
            1+\frac{2(x-\alpha(\theta))}{\alpha(\theta)} &x \in ~]\alpha(\theta),1+\alpha(\theta)]\\
        \end{cases},
    \end{align}
    % and
    % \begin{align}
    % \label{eq:overline_varphi_extended}
    %     \overline{\varphi}(x,\theta) = \begin{cases}
    %         0 &x \in [0,\alpha(\theta)]\\
    %         1+\frac{2(x-\alpha(\theta))}{\alpha(\theta)} &x \in ~]\alpha(\theta),1+\alpha(\theta)]\\
    %     \end{cases},
    % \end{align}
    a similar form as the functions $\overline{\psi}$ and $\overline{\varphi}$ be defined as in~\eqref{eq:overline_psi}. %and~\eqref{eq:overline_varphi}.

    \begin{mylemma}
    \label[lemma]{lemma:upper_bound_on_the_ratio_at_infinity_extended}
Let $\overline{\psi}(x,\theta)$ and $\overline{\varphi}(x,\theta)$ be defined as in~\eqref{eq:overline_psi_extended}. Then, it holds that 
    \begin{align*}
 \frac{\bE \left[\bV \left[ \overline{\varphi} \left(\alpha(\hat{\Theta})- \innerproduct{\hat{A}}{\Theta}, \hat{\Theta} \right) \mid\hat{\Theta} \right]\right]}{\bE\left[\bV \left[ \overline{\psi} \left(\alpha(\hat{\Theta})- \innerproduct{\hat{A}}{\Theta},\hat{\Theta} \right) \mid\hat{\Theta} \right]\right]} \leq \frac{9}{\alpha}. 
\end{align*}
\end{mylemma}
\begin{proof}
The proof follows the one for~\Cref{lemma:upper_bound_on_the_ratio_at_infinity}. Starting with $\bE \left[\bV \left[ \overline{\psi} \left(\alpha(\hat{\Theta})\smash{-}\innerproduct{\hat{A}}{\Theta},\hat{\Theta} \right) \smash{\mid}\hat{\Theta} \right]\right]$, we note that $\overline{\psi} \left(\alpha(\hat{\Theta})-\innerproduct{\hat{A}}{\Theta},\hat{\Theta} \right)$ is equal to $1$ if $\innerproduct{\hat{A}}{\Theta}<0$ and is equal to $0$ otherwise. It can then be written as $\bE[Q(\hat{A})(1-Q(\hat{A}))]$ using the notations from the proof of~\Cref{lemma:upper_bound_on_the_ratio_at_infinity}. Similarly for $\bE \left[\bV \left[ \overline{\varphi} \left(\alpha(\hat{\Theta})-\innerproduct{\hat{A}}{\Theta},\hat{\Theta} \right) \mid\hat{\Theta} \right]\right]$, we distinguish two cases: either $\innerproduct{\hat{A}}{\Theta}\geq0$ and the function is equal to $0$ or $\innerproduct{\hat{A}}{\Theta}<0$ and the function is equal to $1-2\frac{\innerproduct{\hat{A}}{\Theta}}{\innerproduct{\hat{A}}{\hat{\Theta}}}$. Using a similar decomposition as in the proof of~\Cref{lemma:upper_bound_on_the_ratio_at_infinity}, we write $\bE \left[\bV \left[ \overline{\varphi} \left(\alpha(\hat{\Theta})-\innerproduct{\hat{A}}{\Theta},\hat{\Theta} \right) \mid\hat{\Theta} \right]\right]$ as 
\begin{align*}
    \bE&\left[I(\innerproduct{\hat{A}}{\Theta})\left(1-2\frac{\innerproduct{\hat{A}}{\Theta}}{\innerproduct{\hat{A}}{\hat{\Theta}}} \right)^2\right] -\bE\left[\bE\left[I(\innerproduct{\hat{A}}{\Theta})\left(1-2\frac{\innerproduct{\hat{A}}{\Theta}}{\innerproduct{\hat{A}}{\hat{\Theta}}} \right)\right]^2\right],
    \end{align*}
where in the second term, the outer expectation is on $\hat{A},\hat{\Theta}$, and the inner expectation is on $\Theta$. Then taking the supremum over the possible values of $1-2\frac{\innerproduct{\hat{A}}{\Theta}}{\innerproduct{\hat{A}}{\hat{\Theta}}}$ which ranges from $[1-2/\alpha,1+2/\alpha]$ we get: 
\begin{align*}
    \bE \left[\bV \left[ \overline{\varphi} \left(\alpha(\hat{\Theta})-\innerproduct{\hat{A}}{\Theta},\hat{\Theta} \right) \mid\hat{\Theta} \right]\right] &\leq \sup_{\zeta \in [1-2/\alpha,1+2/\alpha]} \bE\left[I(\innerproduct{\hat{A}}{\Theta})\zeta^2\right] - \bE\left[\bE\left[I(\innerproduct{\hat{A}}{\Theta})\zeta \right]^2\right]\\
    &= (1+2/\alpha)^2 \cdot \bE[Q(\hat{A})(1-Q(\hat{A}))].
\end{align*}
Finally, as $\alpha \in [0,1]$, we can upper bound $(1+2/\alpha)^2$ by $9 \alpha^{-2}$ and we conclude the proof.

% We start by introducing the notation $I(x,\theta) \coloneqq \mathds{1}_{\{x < \alpha(\theta)\}}$ and adapt the definition of $Q(\hat{\Theta})\coloneqq \bE[I(\innerproduct{\hat{A}}{\Theta},\hat{\Theta})]$ and $G(\hat{\Theta})\coloneqq \bE[I(\innerproduct{\hat{A}}{\Theta},\hat{\Theta})\frac{\innerproduct{\hat{A}}{\Theta}}{\alpha(\hat{\Theta})}]$. With those definitions, we note that $\bE \left[\bV \left[ \overline{\psi} \left(\smash{\alpha(\hat{\Theta})- }\innerproduct{\hat{A}}{\Theta},\hat{\Theta} \right) \mid\hat{\Theta} \right]\right]$ is equal to the expected variance of Bernoulli random variable with probability given by $Q(\hat{\Theta})$ and can therefore be written as
% $\bE[Q(\hat{\Theta})(1-Q(\hat{\Theta}))]$. The analysis of $\bE \left[\bV \left[ \overline{\varphi} \left(\alpha(\hat{\Theta})-\innerproduct{\hat{A}}{\Theta},\hat{\Theta} \right) \mid\hat{\Theta} \right]\right]$
\end{proof}

%\subsection{Illustrations}
%\label{subsec:illustrations}

\section{Numerical simulations}
\label{sec:numerical_simulations}
To illustrate the improvement of our regret analysis compared to previous work, we performed numerical experiments on a synthetic problem. We considered a logistic bandit problem in dimension $d=10$, with time horizon $T=200$, and with parameter $\beta$ ranging between $[0.25,10]$. For both action space and parameter space, we used the closed $d$-dimensional unit sphere, $\cA = \cO = \mathbf{S}_d(0,1)$ and assumed a uniform prior distribution for the parameter $\Theta$. We computed the expected regret of the Thompson Sampling algorithm using an MCMC %Markov Chain Monte Carlo 
method %(see \Cref{rem:MCMC})
and compared it to three Bayesian regret bounds that hold for continuous spaces: our \Cref{corollary_alpha_one}, \citet[Proposition 4]{russo_learning_2014-1}, and \citet[Proposition 17]{dong_performance_2019} which can be adapted to be compatible with \citet[Theorem 1]{dong_information-theoretic_2018} (see~\Cref{sec:adaptation}).\\

The results are presented in~\Cref{fig:numerical_exp}. The left sub-figure shows the evolution of the expected regret and the regret bounds for two different values of $\beta = \{2,4\}$. For both values of $\beta$, our bound is tighter throughout the entire time horizon and is less sensitive to increasing $\beta$ compared to \citet[Proposition 4]{russo_learning_2014-1} and \Cref{prop:adaptation} adapting \citet[Theorem 2]{dong_information-theoretic_2018}. This behavior is illustrated in the right sub-figure, were the different regret bounds at $t=200$ are compared for values of $\beta$ ranging between $[0.25,10]$. Our bound remains competitive across all values of $\beta$ and is tighter than \citet[Proposition 17]{dong_performance_2019} for values of $\beta\geq2$. Importantly, we observe that while our bound increases only logarithmically, both \citet[Proposition 4]{russo_learning_2014-1} and  \citet[Proposition 17]{dong_performance_2019} increase exponentially with $\beta$. % We note that the actual expected regret decreases for larger $\beta$. This was anticipated since, for large values of $\beta$, the distinction between near-optimal and suboptimal actions becomes more pronounced, facilitating the identification of near-optimal actions.
\begin{figure}[ht]
    \centering
    \scalebox{0.82}{
    \includegraphics[]{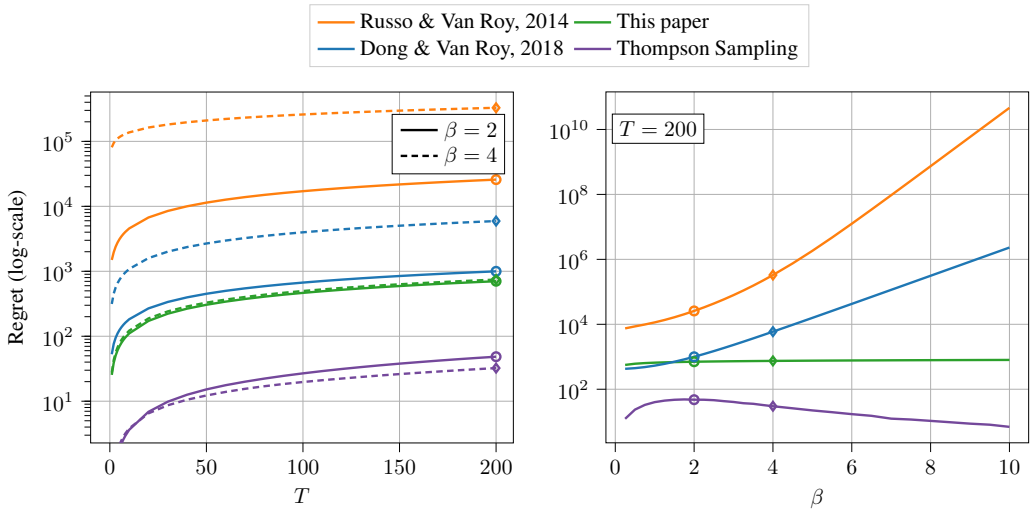}
    }
    \caption{Comparison of Bayesian regret bounds for the logistic bandit setting ($d=10$, $T=200$, $\beta \in [0.25,10]$, $\cA = \cO = \mathbf{S}_d(0,1)$, uniform prior on $\Theta$). The left sub-figure compares the evolution of the bounds and the expected regret for $\beta \in \{2,4\}$. The right sub-figure illustrates the behavior of the bounds and the expected regret at time $T=200$ for values of $\beta$ ranging in $[0.25,10]$. %For computing the TS expected regret, we used the MCMC method with $200$ particles and averaged over $500$ draws.
    }
    \label{fig:numerical_exp}
\end{figure}
% \begin{remark}
%     \label{rem:MCMC}
%     For the logistic bandit setting, we can derive analytically an update rule for the posterior $\bP_t(\Theta)$ given a new reward $R_t$ and the corresponding action $A_t$, using Bayes' rule:
%         \begin{align*}
%         \frac{\bP_t(\Theta|R_t,A_t)}{\bP_t(\Theta)} = \frac{\bP_t(R_t|\Theta,A_t)}{\bP_t(R_t|A_t)} \propto \bP_t(R_t|\Theta,A_t)=\begin{cases}
%                 \phi_\beta(\innerproduct{\Theta}{A_t}) \quad \quad \quad \text{if } R_t=1\\
%                 1-\phi_\beta(\innerproduct{\Theta}{A_t}) \quad \text{ if } R_t=0
%             \end{cases},
%         \end{align*}
%         where the first equality follows from $\bP_t(\Theta)=\bP_t(A_t)$ for Thompson Sampling. For finite parameter spaces, it suffices to update the probability associated with each $\theta\in\cO$ proportional to $\bP_t(R_t|\Theta,A_t)$ and normalize afterwards. For continuous spaces, numerical methods such as Markov Chain Monte Carlo (MCMC) or Gibbs Sampling can be used to approximate the distribution of the posterior.
% \end{remark}

\section{Adapting information ratio bounds from linear to logistic bandits}
\label{sec:adaptation}
This section presents one way to rigorously adapt the TS information bound from \citet[Proposition 17]{dong_performance_2019} to the \emph{one-step compressed Thompson Sampling} such that it can be combined with the regret bound from \citet[Theorem 1]{dong_information-theoretic_2018}.

\begin{myproposition}[{\citet[Proposition 17]{dong_performance_2019}} extended to \emph{one-step compressed TS}]
    \label[proposition]{prop:adaptation}

    For all $\beta>0$, under the logistic bandit setting with logistic function $\phi_{\beta}(x)$,  letting $\tilde{\Theta}_t^\star$ and $\tilde{\Theta}_t$ satisfy the conditions in~\citet[Proposition 2]{dong_information-theoretic_2018}, the \emph{one-step compressed TS} information ratio is bounded as 
    \begin{align*}
     \tilde{\Gamma}(\tilde{\Theta}_t^\star,\tilde{\Theta}_t) \coloneqq \frac{\bE_t[ R(\pi_\star(\tilde{\Theta}_t^\star),\Theta) - R(\pi_\star(\tilde{\Theta}_t),\Theta)]^2}{\mi_t(\tilde{\Theta}_t^\star; R(\pi_\star(\tilde{\Theta}_t),\Theta),\tilde{\Theta}_t)} \leq  d \frac{(1+\exp(\beta))^4}{32 \exp(\beta)^2}.   
    \end{align*}
\end{myproposition}

\begin{proof}
We combine the proof techniques from~\citet[Proposition 17]{dong_performance_2019} and~\cite[Proposition 3]{dong_information-theoretic_2018}. To simplify the exposure we will omit the subscript $t$ and reuse the notations introduced in~\Cref{sec:analysis}. With these notations, the one-step compressed information ratio can be written as 
\begin{align*}
    \frac{\bE[ \Bern{\pi_\star(\tilde{\Theta}^\star)}{\Theta} - \Bern{\pi_\star(\tilde{\Theta})}{\Theta}]^2}{\mi(\tilde{\Theta}^\star; \Bern{\pi_\star(\tilde{\Theta})}{\Theta},\tilde{\Theta})}.
\end{align*}

Similarly to~\citet[Proposition 17]{dong_performance_2019}, we let $L_1=\inf_{x\in[-1,1]} |\phi_\beta(x)'| = \frac{\beta \exp(\beta)}{(1+\exp(\beta))^2}$ and $L_2=\sup_{x\in[-1,1]} |\phi_\beta(x)'| = \frac{\beta}{4}$. We will start by analyzing the numerator which we write as

\begin{align*}
    \bE[ \Bern{\pi_\star(\tilde{\Theta}^\star)}{\Theta}\smash{-}  \Bern{\pi_\star(\tilde{\Theta})}{\Theta}]^2 &=  \bE[ \PhiBeta{\pi_\star(\tilde{\Theta}^\star)}{\Theta} - \PhiBeta{\pi_\star(\tilde{\Theta})}{\Theta}]^2\\
    &\leq L_2^2  \cdot \bE[ \innerproduct{\pi_\star(\tilde{\Theta}^\star)}{\Theta} - \innerproduct{\pi_\star(\tilde{\Theta})}{\Theta}]^2, 
\end{align*}
where the inequality follows from $\frac{\phi_\beta(x_1)-\phi_\beta(x_2)}{x_1-x_2} \leq \sup_{x\in[-1,1]} |\phi_\beta(x)'|$ for all $x_1,x_2 \in [-1,1]$. \\

We then focus on the denominator, which we lower bound similarly as in~\cite[Lemma 2]{dong_information-theoretic_2018}. Introducing the notation $R(\tilde{\theta}) = \Bern{\pi_\star(\tilde{\theta})}{\Theta}$, we can write

\begin{align*}
\mi(\tilde{\Theta}^\star; \Bern{\pi_\star(\tilde{\Theta})}{\Theta},\tilde{\Theta}) &\overset{(i)}{=} \mi(\tilde{\Theta}^\star;\tilde{\Theta})+\mi(\tilde{\Theta}^\star; \Bern{\pi_\star(\tilde{\Theta})}{\Theta}|\tilde{\Theta})\\
&\overset{(j)}{=} \sum_{\tilde{\theta} \in \cO_\epsilon} \sum_{\tilde{\theta}^\star \in \cO_\epsilon} \bP[\tilde{\Theta}\smash{=} \tilde{\theta}]  \bP[\tilde{\Theta}^\star\smash{=}\tilde{\theta}^\star] \KL{\bP_{R(\tilde{\theta})|\tilde{\Theta}^\star = \tilde{\theta}^\star}}{\bP_{R(\tilde{\theta})}}  \\
&\overset{(k)}{\geq} \sum_{\tilde{\theta} \in \cO_\epsilon} \sum_{\tilde{\theta}^\star \in \cO_\epsilon} \bP[\tilde{\Theta} \smash{=}\tilde{\theta}]  \bP[\tilde{\Theta}^\star\smash{=} \tilde{\theta}^\star] 2 \big(\bE[R(\tilde{\theta})|\tilde{\Theta}^\star\smash{=} \tilde{\theta}^\star]-\bE[R(\tilde{\theta})]\big)^2,\\
%&= \bE[ (\PhiBeta{\pi_\star(\tilde{\Theta}^\star)}{\Theta} - \PhiBeta{\pi_\star(\tilde{\Theta})}{\Theta}]^2\\
%&\overset{(l)}{\geq} 2 L_1^2 \cdot \bE[(\innerproduct{\pi_\star(\tilde{\Theta}^\star)}{\Theta} - \innerproduct{\pi_\star(\tilde{\Theta})}{\Theta})^2]
\end{align*}
where $\cO_\epsilon$ is the set of values for the random variables $\tilde{\theta}^\star$ and  $\tilde{\Theta}_t$ as defined in~\citet[Proposition 2]{dong_information-theoretic_2018}. The inequality (i) follows from the chain-rule; (j) follows from $\tilde{\Theta}_t^\star$and $\tilde{\Theta}_t$ being independent as they satisfy the conditions in~\citet[Proposition 2]{dong_information-theoretic_2018}; (k) is obtained using the Donsker-Varadhan inequality~\citep[Theorem~5.2.1]{gray_entropy_2013} as in~\citet[Lemma 3]{russo_information-theoretic_2015}. We then continue to lower bound
\begin{align*}
    \small{(k)}&=2\sum_{\tilde{\theta} \in \cO_\epsilon} \sum_{\tilde{\theta}^\star \in \cO_\epsilon} \bP[\tilde{\Theta} \smash{=}\tilde{\theta}]  \bP[\tilde{\Theta}^\star\smash{=} \tilde{\theta}^\star]  \big(\bE[\PhiBeta{\pi_\star(\tilde{\theta})}{\Theta}|\tilde{\Theta}^\star\smash{=} \tilde{\theta}^\star]-\bE[\PhiBeta{\pi_\star(\tilde{\theta})}{\Theta}]\big)^2\\
    &\geq 2 L_1^2\sum_{\tilde{\theta} \in \cO_\epsilon} \sum_{\tilde{\theta}^\star \in \cO_\epsilon} \bP[\tilde{\Theta} \smash{=}\tilde{\theta}]  \bP[\tilde{\Theta}^\star\smash{=} \tilde{\theta}^\star]  \big(\bE[\innerproduct{\pi_\star(\tilde{\theta})}{\Theta}|\tilde{\Theta}^\star\smash{=} \tilde{\theta}^\star]-\bE[\innerproduct{\pi_\star(\tilde{\theta})}{\Theta}]\big)^2,
\end{align*}
using the fact that $\frac{\phi_\beta(x_1)-\phi_\beta(x_2)}{x_1-x_2} \geq \inf_{x\in[-1,1]} |\phi_\beta(x)'|$ for all $x_1,x_2 \in [-1,1]$.\\
Combining the inequality on the numerator and the denominator, we get that
\begin{align}
\label{eq:adaptation}
    \tilde{\Gamma}(\tilde{\Theta}_t^\star,\tilde{\Theta}_t) \leq \frac{L_2^2}{2 L_1^2} \frac{\Big(\sum_{\theta' \in \cO_\epsilon}  \bP[\tilde{\Theta} \smash{=}\theta']    \big(\bE[\innerproduct{\pi_\star(\theta')}{\Theta}|\tilde{\Theta}^\star\smash{=} \theta']-\bE[\innerproduct{\pi_\star(\theta')}{\Theta}]\big)\Big)^2}{\sum_{\tilde{\theta} \in \cO_\epsilon} \sum_{\tilde{\theta}^\star \in \cO_\epsilon} \bP[\tilde{\Theta} \smash{=}\tilde{\theta}]  \bP[\tilde{\Theta}^\star\smash{=} \tilde{\theta}^\star]  \big(\bE[\innerproduct{\pi_\star(\tilde{\theta})}{\Theta}|\tilde{\Theta}^\star\smash{=} \tilde{\theta}^\star]-\bE[\innerproduct{\pi_\star(\tilde{\theta})}{\Theta}]\big)^2}.
\end{align}
Without loss of generality, we write $\cO_\epsilon = \{\tilde{\theta}_1,\ldots,\tilde{\theta}_{|\cO_\epsilon|} \}$. Now, we define the random matrix $M \in \bR^{|\cO_\epsilon|\times|\cO_\epsilon|}$ where for each $i,j\in \{1,\ldots,|\cO_\epsilon|\}$ the corresponding entry is given by 
\begin{align*}
    M_{i,j} = \sqrt{\bP[\tilde{\Theta}^\star\smash{=} \tilde{\theta}_i]  \bP[\tilde{\Theta}^\star\smash{=} \tilde{\theta}_j]} \big(\bE[\innerproduct{\pi_\star(\tilde{\theta}_i)}{\Theta}|\tilde{\Theta}^\star\smash{=} \tilde{\theta}_j]-\bE[\innerproduct{\pi_\star(\tilde{\theta}_i)}{\Theta}]\big).
\end{align*}
We note that we can rewrite \cref{eq:adaptation} using the trace and the Frobenius norm of the matrix $M$, as 
\begin{align*}
    \tilde{\Gamma}(\tilde{\Theta}_t^\star,\tilde{\Theta}_t) \leq \frac{L_2^2}{2 L_1^2} \frac{\textnormal{Trace}(M)^2}{||M||_F} \leq \frac{L_2^2}{2 L_1^2} \textnormal{Rank(M)}, 
\end{align*}
where the last inequality is obtained from \cite[Fact 10]{russo_information-theoretic_2015}. 

The proof concludes showing the rank of the matrix $M$ is upper bounded by $d$. 
    For the sake of brevity, we define $\bar{\Theta} \coloneqq \bE[\Theta)]$ and $Q_i \coloneqq \bE[\Theta|\tilde{\Theta}^\star\smash{=}\tilde{\theta}_i]$ for all $i\in \{1,\ldots,|\cO_\epsilon| \}$. We then have $\bE[\innerproduct{\pi_\star(\tilde{\theta}_i)}{\Theta}|\tilde{\Theta}^\star\smash{=} \tilde{\theta}_j] = \innerproduct{\pi_\star(\tilde{\theta}_i)}{Q_j}$  and $ \bE[\innerproduct{\pi_\star(\tilde{\theta}_i)}{\Theta}] = \innerproduct{\pi_\star(\tilde{\theta}_i)}{\bar{\Theta}}$ . Since the inner product is linear, we can rewrite each entry $M_{i,j}$ of the matrix $M$ as
    \begin{align*}
        \sqrt{\bP[\tilde{\Theta}^\star\smash{=} \tilde{\theta}_i]\bP[\tilde{\Theta}^\star\smash{=} \tilde{\theta}_j]}\innerproduct{\pi_\star(\tilde{\theta}_i)}{Q_j-\bar{\Theta}}.
    \end{align*}
    Equivalently, the matrix $M$ can be written as
    \begin{align*}
        \begin{bmatrix}
        \sqrt{\bP[\tilde{\Theta}^\star\smash{=} \tilde{\theta}_1]}\pi_\star(\tilde{\theta}_i)\\
        \vdots\\
        \sqrt{\bP[\tilde{\Theta}^\star\smash{=} \tilde{\theta}_{|\cO_\epsilon|}]}\pi_\star(\tilde{\theta}_i)
        \end{bmatrix}
        \begin{bmatrix}
        \sqrt{\bP[\tilde{\Theta}^\star\smash{=} \tilde{\theta}_1]}(Q_1-\bar{\Theta})&
        \ldots &&
        \sqrt{\bP[\tilde{\Theta}^\star\smash{=} \tilde{\theta}_{|\cO_\epsilon|}]}(Q_{|\cO_\epsilon|}-\bar{\Theta})
        \end{bmatrix}.
    \end{align*}
    This rewriting highlights that $M$ can be written as the product of a $|\cO_\epsilon|$ by $d$ matrix and a $d$ by $|\cO_\epsilon|$ matrix and therefore has a rank lower or equal than $\min(d,|\cO_\epsilon|)$.
\end{proof}

\section{Regarding the gaps in previous literature}
\label{sec:arguments_gaps_in_dong}

In~\Cref{sec:introduction}, we mention that the main results of~\cite{dong_performance_2019} are incomplete because of two shortcomings. The first one concerns a gap in their analysis of the Thompson Sampling information ratio for values of $\beta>2$. The second one regards a mistake in their regret analysis, which combines incompatible results. We elaborate on both shortcomings below.

\paragraph{Regarding the first shortcoming} 
We identified an issue in the proof of~\citet[Theorem 5]{dong_performance_2019} at the end of the proof on page 20, where the inequality $\chi > \xi > 0.1 \lambda$ is stated without justification. This inequality plays a crucial role in deriving their bound on the Thompson Sampling information ratio, but the only evidence provided is~\citet[Figure 3]{dong_performance_2019}, which illustrates the functions $\chi(\lambda,\beta)$ and $\xi(\lambda,\beta)$ for the specific case of $\beta = 2$. While this figure suggests that the inequality holds for $\beta = 2$, it cannot be used to conclude that the inequality holds in general for $\beta \geq 2$. Additionally, we note that the computation of $\chi(\lambda,\beta)$ and $\xi(\lambda,\beta)$ for given values of $\lambda$ and $\beta$ is highly intricate, and despite our best efforts, we were unable to reproduce~\citet[Figure 3]{dong_performance_2019}.

\paragraph{Regarding the second shortcoming}  As mentioned earlier, the regret analysis in~\citet[Theorems 1 and 5]{dong_performance_2019} combines incompatible results. Specifically, the paper uses a uniform bound on the information ratio of the ``standard" Thompson Sampling (provided in~\citet[Appendix B, Eq. (18)]{dong_performance_2019}) together with~\citet[Theorem 1]{dong_information-theoretic_2018}, which requires a uniform bound on the information ratio of the \emph{one-step compressed Thompson Sampling}. This inconsistency invalidates the regret bounds derived in~\citet[Theorems 1 and 5]{dong_performance_2019}. We emphasize that the problem is ``hidden" in~\citet[Proposition 9]{dong_performance_2019}, which incorrectly restates~\citet[Theorem 4]{dong_information-theoretic_2018}. It is important to note that there is no straightforward way to extend the result from~\citet[Theorem 4]{dong_information-theoretic_2018} such that it the regret bound works with bounds on the ``standard" Thompson Sampling (and not with \emph{one-step compressed Thompson Sampling}), as this compressed version of Thompson Sampling is central to~\citet[Theorem 4]{dong_information-theoretic_2018}.  %Similarly, we want to highlight that the techniques used to analyze the information ratio in~\citet[Proof of Proposition 5]{dong_performance_2019} do not apply to the \emph{one-step compressed Thompson Sampling} information ratio. For instance,~\citet[Proof of Proposition 15]{dong_performance_2019} requires a one-to-one mapping between parameters and optimal action (see eq. (22), inequalities (c) and (e)). However, by definition, the equivalent requirement for \emph{one-step compressed Thompson Sampling} (a one-to-one mapping between the statistic $\psi$ and the optimal action) cannot hold, as by construction the $\psi$ is constructed to be less informative than the parameter $\theta$.

While it is possible to use~\citet[Proposition 1]{dong_information-theoretic_2018} directly with a uniform bound on the ``standard" Thompson Sampling information ratio, this approach is limited because~\citet[Proposition 1]{dong_information-theoretic_2018} provides a loose bound. Specifically, this bound depends on the cardinality of the parameter space $\Theta$ through the entropy $\mathsf{H}(\Theta^\star)$ (or on the cardinality of the action space through $\mathsf{H}(A^\star)$ in the original version~\citet[Proposition 1]{russo_information-theoretic_2015}). This issue is highlighted in~\cite{dong_information-theoretic_2018} at the end of Section 3, and serves as a motivation for the introduction of the \emph{one-step-compressed Thompson Sampling} regret analysis in the paper. Combining~\citet[Proposition 1]{dong_information-theoretic_2018} with our bound on the Thompson Sampling information ratio,~\Cref{prop:info_ratio_logistic_bandits}, results in a regret bound of the order $O(\sqrt{d T \log(|\mathcal{O}|)})$, which becomes vacuous for infinite or continuous parameter spaces.
        
%We emphasize that the \emph{one-step compressed Thompson Sampling} information ratio is a fundamentally different quantity and is significantly more challenging to analyze due to the intricate construction of the \emph{one-step compressed Thompson Sampling} (c.f.~\citet[Proof of Proposition 2]{dong_information-theoretic_2018}). Notably, the techniques used to analyze the information ratio in~\cite{dong_performance_2019}, even for the case where $\beta\leq 2$, do not apply to the \emph{one-step compressed Thompson Sampling} information ratio. For instance,~\citet[Proof of Proposition 5]{dong_performance_2019} requires a one-to-one mapping between parameters and optimal action (see~\citet[Proof of Proposition 15, eq. (22), inequalities (c) and (e)]{dong_performance_2019}). However, by definition, the equivalent requirement for \emph{one-step compressed Thompson Sampling} (a one-to-one mapping between the statistic $\psi$ and the optimal action) cannot hold, as by construction the $\psi$ is constructed to be less informative than the parameter $\theta$.   

\section{Constructing $\pi_\star$ as a one-to-one mapping}
\label{sec:constructing_the_one_to_one_mapping}

If the mapping $\pi_\star(\theta)$ is not one-to-one, it could either be that a particular parameter is optimal for several actions or that a particular action is optimal for several parameters.

In the first case, for example, if actions $a_1, a_2 \in \mathcal{A}$ are optimal for the same parameter $\theta_1 \in \mathcal{O}$, then it implies that $\mathbb{E}[R(a_1, \theta_1)] = \mathbb{E}[R(a_2, \theta_1)] \geq \mathbb{E}[R(a, \theta_1)]$ for all $a \in \mathcal{A}$ with $a \neq a_1$ and $a \neq a_2$. In this scenario, we can arbitrarily set $\pi_\star(\theta_1) = a_1$ without affecting the regret of Thompson Sampling, as $\mathbb{E}[R(a_1, \theta_1)] = \mathbb{E}[R(a_2, \theta_1)]$.

In the second case, if an action $a_1 \in \mathcal{A}$ is optimal for multiple parameters, say $\theta_1, \theta_2 \in \mathcal{O}$, we can artificially construct an \emph{action label} set $\mathcal{A}'$ such that two labels, $a_1', a_2' \in \mathcal{A}'$, are associated with $a_1$. For all other actions in $\mathcal{A} \setminus \{a_1\}$, there is a corresponding label in $\mathcal{A}'$. We denote the mapping between action labels and their corresponding actions using the function $\rho: \mathcal{A}' \to \mathcal{A}$.
We can construct a function $\pi_\star: \mathcal{O} \to \mathcal{A}'$ such that $\pi_\star$ is a one-to-one mapping between the parameters $\mathcal{O}$ and the action labels $\mathcal{A}'$. We define the \emph{optimal action label} as ${A^\star}' = \pi_\star(\Theta)$ and the \emph{Thompson Sampling action label} as ${A_t}' = \pi_\star(\Theta_t)$. 
This artificial construction, illustrated in~\Cref{fig:construction_of_one_to_one}, is intended solely for the purposes of our regret analysis and has no impact on the regret of Thompson Sampling. The instant regret of Thompson Sampling at time $t \in \{1, \ldots, T\}$ remains $R(A^\star, \Theta) - R(A_t, \Theta)$, where $A^\star = \rho({A^\star}')$ is still the optimal action for $\Theta$, and $A_t = \rho(A_t')$ is the action selected by Thompson Sampling for $\Theta_t$. In this context, the Thompson Sampling information ratio would be adapted and defined as: \begin{equation*} \Gamma_t \coloneqq \frac{\bE_t[R(A^\star, \Theta) - R(A_t, \Theta)]^2}{\mi_t({A^\star}'; R(A_t, \Theta), A_t')}, \end{equation*} representing the ratio between the current squared regret and the information gathered about the optimal action label. One can verify that the analysis of the information ratio in~\Cref{sec:analysis},~\Cref{sec:ratio_exp_variances_appendix}, and~\Cref{sec:extension}, proceeds the same with this adapted definition and leads to the same upper bound.

\begin{figure}[ht]
    \centering
    \includegraphics[]{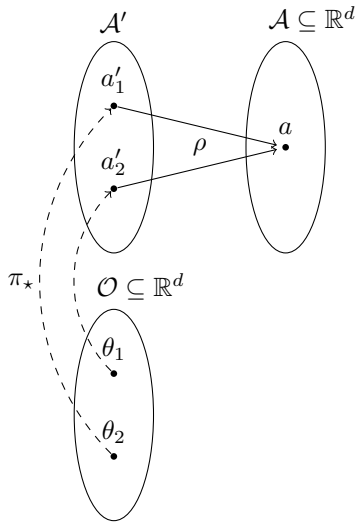}
    \caption{Illustration of the artificial construction of the action label set $\cA'$.}
    \label{fig:construction_of_one_to_one}
\end{figure}

\end{document}